\documentclass{article}

\usepackage[table]{xcolor}
\usepackage{hyperref}

\usepackage[accepted]{icml2021}

\usepackage{wrapfig}
\usepackage[utf8]{inputenc}
\usepackage{booktabs}
\usepackage{multirow}
\usepackage{float}
\usepackage{amssymb}
\usepackage{amsmath}
\usepackage{amsthm}
\usepackage{bbold}
\usepackage{siunitx}
\usepackage{tikz}
\usepackage{nicefrac}
\usepackage{chngcntr}
\usepackage{verbatim}
\usepackage{mathtools}
\usepackage{esvect}
\usepackage{xspace}
\usepackage{xparse}
\usepackage{algorithm}
\usepackage{algorithmic}
\usepackage{thmtools,thm-restate}
\usepackage{graphicx}
\usepackage{sidecap}
\usepackage{hyperref}
\hypersetup{breaklinks=true, colorlinks, linkcolor=red, citecolor=blue}      
\usepackage{enumitem}
\usepackage[flushleft]{threeparttable}
\usepackage{tikzit}
\usepackage{subcaption}

\tikzstyle{server}=[fill={rgb,255: red,64; green,64; blue,64}, draw=black, shape=circle]
\tikzstyle{client}=[fill={rgb,255: red,55; green,126; blue,184}, draw=black, shape=circle]
\tikzstyle{optima}=[fill={rgb,255: red,55; green,126; blue,184}, draw=black, shape=rectangle]
\tikzstyle{globalopt}=[fill={rgb,255: red,255; green,128; blue,0}, draw=black, shape=rectangle]
\tikzstyle{globalroundopt}=[fill={rgb,255: red,255; green,128; blue,0}, draw=black, shape=circle]
\tikzstyle{pseudoserver}=[fill={rgb,255: red,191; green,191; blue,191}, draw={rgb,255: red,128; green,128; blue,128}, shape=circle]
\tikzstyle{sgdnode}=[fill=white, draw=black, shape=circle]
\tikzstyle{client2}=[fill={rgb,255: red,77; green,175; blue,74}, draw=black, shape=circle]
\tikzstyle{clientopt}=[fill={rgb,255: red,77; green,175; blue,74}, draw=black, shape=rectangle]
\tikzstyle{averageopt}=[fill={rgb,255: red,191; green,191; blue,191}, draw={rgb,255: red,128; green,128; blue,128}, shape=rectangle]
\tikzstyle{averageroundopt}=[fill={rgb,255: red,191; green,191; blue,191}, draw={rgb,255: red,128; green,128; blue,128}, shape=circle]
\tikzstyle{gradient}=[dashed]

\tikzstyle{server update}=[->, draw={rgb,255: red,77; green,175; blue,74}, very thick]
\tikzstyle{client update}=[draw={rgb,255: red,55; green,126; blue,184}, ->, very thick]
\tikzstyle{update}=[draw=gray, dashed, -, very thick, very thick]
\tikzstyle{sgd}=[draw={rgb,255: red,65; green,61; blue,58}, ->, very thick]
\tikzstyle{correction}=[draw={rgb,255: red,228; green,26; blue,28}, ->, very thick]
\tikzstyle{midpoint}=[-, draw={rgb,255: red,191; green,191; blue,191}, very thick]

\setlist{leftmargin=5.5mm}
\DeclareCaptionFormat{myformat}{\fontsize{8.5}{10}\selectfont#1#2#3}
\usepackage{bm}
\usepackage{tabularx} 
\usepackage{cleveref}

\DeclarePairedDelimiterX{\inp}[2]{\langle}{\rangle}{#1, #2}
\DeclarePairedDelimiterX{\abs}[1]{\lvert}{\rvert}{#1}
\DeclarePairedDelimiterX{\norm}[1]{\lVert}{\rVert}{#1}
\DeclarePairedDelimiterX{\cbr}[1]{\{}{\}}{#1} 
\DeclarePairedDelimiterX{\rbr}[1]{(}{)}{#1} 
\DeclarePairedDelimiterX{\sbr}[1]{[}{]}{#1} 

  \DeclareMathOperator{\expect}{\mathbb{E}}
  \DeclareMathOperator{\E}{\expect}
  \DeclareMathOperator{\ind}{\mathbb{1}}

  \DeclareMathOperator{\sgn}{sign}
  \makeatletter
  \def\sign{\@ifnextchar*{\@sgnargscaled}{\@ifnextchar[{\sgnargscaleas}{\@ifnextchar{\bgroup}{\@sgnarg}{\sgn} }}}
  \def\@sgnarg#1{\sgn\rbr{#1}}
  \def\@sgnargscaled#1{\sgn\rbr*{#1}}
  \def\@sgnargscaleas[#1]#2{\sgn\rbr[#1]{#2}}
  \makeatother

  \DeclareMathOperator*{\argmin}{arg\,min}

  \providecommand{\0}{\bm{0}}

  \providecommand{\cc}{\bm{c}}
  \providecommand{\dd}{\bm{d}}
  \providecommand{\ee}{\bm{e}}

  \renewcommand{\gg}{\bm{g}}

  \providecommand{\mm}{\bm{m}}

  \renewcommand{\vv}{\bm{v}}
  
  \providecommand{\xx}{\bm{x}}
  \providecommand{\yy}{\bm{y}}

  \newcommand{\bmu}{\boldsymbol{\mu}}
  \newcommand{\muv}{\bmu}

  \providecommand{\cA}{\mathcal{A}}
  \providecommand{\cB}{\mathcal{B}}

  \providecommand{\cF}{\mathcal{F}}
  \providecommand{\cG}{\mathcal{G}}

  \providecommand{\cM}{\mathcal{M}}
  
  \providecommand{\cO}{\mathcal{O}}
  \providecommand{\cP}{\mathcal{P}}

  \providecommand{\cS}{\mathcal{S}}
  \providecommand{\cT}{\mathcal{T}}

  \providecommand{\cX}{\mathcal{X}}

\newtheorem{theorem}{Theorem}

\newtheorem{corollary}[theorem]{Corollary}

\newtheorem{lemma}{Lemma}
\newtheorem{remark}[lemma]{Remark}

\newtheorem{definition}{Definition}

\newcommand{\ignore}[1]{}

\newcommand{\e}{\epsilon}

\definecolor{color1}{RGB}{228,26,28}
\definecolor{color2}{RGB}{55,126,184}
\definecolor{color3}{RGB}{77,175,74}
\definecolor{color4}{RGB}{152,78,163}
\definecolor{color5}{RGB}{255,127,0}

\usepackage{pifont}
\makeatletter
\newcommand{\myitem}[1]{%
\item[\textbf{(#1)}]\protected@edef\@currentlabel{#1}%
}
\makeatother

\usetikzlibrary{fit,calc}

\colorlet{worker}{red!40}
\newcommand{\speedup}[1]{{\color{gray}(\ifdim #1 pt > 0.3pt #1\else $< #1$\fi{}$\times$)}}

\usepackage{xspace}
\xspaceaddexceptions{]\}}

\newcommand{\Alg}{\text{\sc Alg}}
\renewcommand{\cA}{{\text{\sc Agg}}}

\newcommand{\arxiv}[1]{}

\begin{document}
\twocolumn[
  \icmltitle{Learning from History for Byzantine Robust Optimization}



  \icmlsetsymbol{equal}{*}

  \begin{icmlauthorlist}
    \icmlauthor{Sai Praneeth Karimireddy}{epfl}
    \icmlauthor{Lie He}{epfl}
    \icmlauthor{Martin Jaggi}{epfl}
  \end{icmlauthorlist}

  \icmlaffiliation{epfl}{EPFL, Switzerland}
  \icmlcorrespondingauthor{Sai Praneeth Karimireddy}{sai.karimireddy@epfl.ch}

  \icmlkeywords{Byzantine robustness, Federated Learning, Distributed Learning, Stocastic Optimization}

  \vskip 0.3in
]



\printAffiliationsAndNotice{}  


\begin{abstract}%
  Byzantine robustness has received significant attention recently given its importance for distributed and federated learning. In spite of this, we identify severe flaws in existing algorithms even when the data across the participants is identically distributed. First, we show realistic examples where current state of the art robust aggregation rules fail to converge even in the absence of any Byzantine attackers.
  Secondly, we prove that even if the aggregation rules may succeed in limiting the influence of the attackers in a single round, the attackers can couple their attacks across time eventually leading to divergence. To address these issues, we present two surprisingly simple strategies: a new robust \emph{iterative clipping} procedure, and incorporating \emph{worker momentum} to overcome time-coupled attacks.
  This is the first provably robust method for the standard stochastic optimization setting. Our code is open sourced at \href{https://github.com/epfml/byzantine-robust-optimizer}{this link}\footnote[2]{\url{https://github.com/epfml/byzantine-robust-optimizer}}.
\end{abstract}

\section{Introduction}
\begin{quotation}
  \emph{``Those who cannot remember the past are condemned to repeat it.''} \quad --~George Santayana.
\end{quotation}
Growing sizes of datasets as well as concerns over data ownership, security, and privacy have lead to emergence of new machine learning paradigms such as distributed and federated learning~\citep{kairouz2019federated}. In both of these settings, a central coordinator orchestrates many worker nodes in order to train a model over data which remains decentralized across the workers. While this decentralization improves scalability security and privacy, it also opens up the training process to manipulation by the workers~\citep{lamport2019byzantine}. These workers may be actively malicious trying to derail the process, or might simply be malfunctioning and hence sending arbitrary messages. Ensuring that our training procedure is robust to a small fraction of such potentially malicious agents is termed Byzantine robust learning and is the focus of the current work.

Given the importance of this problem, it has received significant attention from the community with early works including~\citep{feng2014distributed,blanchard2017machine,Chen_2017,yin2018byzantinerobust}. Most of these approaches replace the averaging step of distributed or federated SGD with a robust aggregation rule such as the median. However, a closer inspection reveals that these procedures are quite brittle: we show that there exist realistic scenarios where they fail to converge, even if there are \emph{no Byzantine attackers} and the data distribution is identical across the workers (i.i.d.). This turns out to be because on their excessive sensitivity to the distribution of the noise in the gradients. The impractical assumptions made by these methods are often violated in practice, and lead to the failure of these aggregation rules.

Further, there have been recent state of the art attacks \citep{baruch2019little,xie2019fall} which empirically demonstrate a second source of failure. They show that even when current aggregation rules may succeed in limiting the influence of the attackers in any single round, they may still diverge when run for multiple rounds. We prove that this is inevitable for a wide class of methods---any aggregation rule which ignores history can be made to eventually diverge. This is accomplished by using the inherent noise in the gradients to mask small perturbations which are undetectable in a single round, but accumulate over time.

Finally, we show how to circumvent both the issues outlines above. We first describe a simple new aggregator based on iterative \emph{centered clipping} which is much more robust to the distribution of the gradient noise. This aggregator is especially interesting since, unlike most preceding methods, it is very scalable requiring only $\cO(n)$ computation and communication per round. Further, it is also compatible with other strategies such as asynchronous updates~\citep{chen2016revisiting} and secure aggregation~\citep{bonawitz2017practical}, both of which are crucial for real world applications. Secondly, we show that the time coupled attacks can easily be overcome by using \emph{worker momentum}. Momentum averages the updates of each worker over time, reducing the variance of the good workers and exposing the time-coupled perturbations. We prove that our methods obtain optimal rates, and our theory also sheds light on the role of momentum in decreasing variance and building resilience to Byzantine workers.

\paragraph{Contributions.} Our main results are summarized below.
\begin{itemize}[nosep]
  \item We show that most state of the art robust aggregators require strong assumptions and can fail in real settings even in the complete absence of Byzantine workers.
  \item We prove a strong lower bound showing that any optimization procedure which does not use history will diverge in the presence of time coupled attacks.
  \item We propose a simple and efficient aggregation rule based on iterative clipping and prove its performance under standard assumptions.
  \item We show that using momentum successfully defends against time-coupled attacks and provably converges when combined with any Byzantine robust aggregator.
  \item We incorporate the recent momentum based variance reduction (MVR) with Byzantine aggregators to obtain optimal rates for robust non-convex optimization.
  \item We perform extensive numerical experiments validating our techniques and results.
\end{itemize}

\paragraph{Setup.} Let us formalize the robust non-convex stochastic optimization problem in the presence of a $\delta$ fraction of Byzantine workers.
\begin{definition}[$\delta$-robust non-convex optimization]\label{asm:robust-opt-problem}
  Given some loss function $f(\xx)$, $\e > 0$, and access to $n$ workers we want to find a stationary point $\xx$ such that $\E\norm{\nabla f(\xx)}^2 \leq \e$. The optimization proceeds in rounds where in every round, each worker $i \in [n]$ can compute a stochastic gradient $g_i(\yy)$ at any parameter $\yy$ in parallel. Then, each worker $i \in [n]$ sends some message $\cM_{i,t}$ to the server. The server utilizes these messages to update the parameters and proceeds to the next round. During this process, we will assume that
  \begin{itemize}[nosep]
    \item The function $f$ is $L$-smooth i.e. it satisfies $\norm{\nabla f(\xx) - \nabla f(\yy)} \leq L\norm{\xx -\yy}$ for any $\xx, \yy$, and is bounded from below by $f^\star$.
    \item Each worker $i$ has access to an independent and unbiased stochastic gradient with $\E[g_i(\xx) | \xx] = \nabla f(\xx)$ and variance bounded by $\sigma^2$, $\E\norm{g_i(\xx) - \nabla f(\xx)}^2 \leq \sigma^2$.
    \item Of the $n$ workers, at least $(1- \delta)n$ workers are \emph{good} (denoted by $\cG$) and will follow the protocol faithfully. The rest of the \emph{bad} or Byzantine workers (denoted by $\cB$) may act maliciously and can communicate arbitrary messages to the server.
    \item These Byzantine workers are assumed to omniscient i.e. they have access to the computations made by the rest of the good workers. However, we assume that this set of Byzantine workers $\cB$ \emph{remains fixed} throughout the optimization process.
  \end{itemize}
\end{definition}

\section{Related work}

\paragraph{Robust aggregators.} Distributed algorithms in the presence of Byzantine agents has a long history \citep{lamport2019byzantine} and is becoming increasingly important in modern distribution and federated machine learning \citep{kairouz2019federated}. Most solutions involve replacing the averaging of the updates from the different machines with more robust aggregation rules such as coordinate-wise median method~\cite{yin2018byzantinerobust}, geometric median methods~\citep{blanchard2017machine,Chen_2017,pillutla2019robust}, majority voting~\cite{bernstein2018signsgd,jin2020stochasticsign} etc. There have also been attempts to use recent breakthroughs in robust high-dimensional aggregators~\cite{diakonikolas2018sever,su2018securing,elmhamdi2019fast,data2019data,data2020byzantineresilienta}. However, these latter procedures are computationally expensive (quadratic in dimensions per round) and further it is unclear if the improved guarantees for mean estimation translate to improved performance in the distributed machine learning settings.
Finally, for most of the above approaches, convergence guarantees when provided rely on using an extremely large batch size or strong unrealistic assumptions making them practically irrelevant.

Other more heuristic approaches propose to use a penalization or reweighting of the updates based on reputations ~\cite{peng2020robust,li2019rsa,fu2019attackresistant,regatti2020befriending,rodriguezbarroso2020dynamic}. These schemes however need to trust that all workers report correct statistics. In such settings where we have full control over the workers (e.g. within a datacenter) coding theory based solutions which can correct for the mistakes have also been proposed~\cite{chen2018draco,rajput2019detox,gupta2019randomized,konstantinidis2020byzshield,data2018data,data2019data}. These however are not applicable in federated learning where the data is decentralized across untrusted workers.

\paragraph{Time coupled attacks and defenses.} Recently, two state-of-the-art attacks have been proposed which show that the state of the art Byzantine aggregation rules can be easily circumvented~\citep{baruch2019little,xie2019fall}. The key insight is that while the robust aggregation rules may ensure that the influence of the Byzantine workers in any single round is limited, the attackers can couple their attacks across the rounds. This way, over many training rounds the attacker is able to move weights significantly away from the desired direction and thus achieve the goal of lowering the model quality. Defending against time-coupled attacks and showing provable guarantees is one of the main concerns of this work.

It is clear that time-coupled attacks need time-coupled defenses. Closest to our work is that of \citet{alistarh2018byzantine} who use martingale concentration across the rounds to give optimal Byzantine robust algorithms for convex functions. However, this algorithm is inherently not applicable to more general non-convex functions. The recent independent work of \citet{allenzhu2021byzantineresilient} extend the method of \citet{alistarh2018byzantine} to non-convex functions as well. However, they assume that the noise in stochastic gradients is bounded almost surely instead of the more standard assumption that only the variance is bounded. Theoretically, such strong assumptions are unlikely to hold~\citep{zhang2019adam} and even Gaussian noise is excluded. Further, the lower-bounds of \citep{arjevani2019lower} no longer apply, and thus their algorithm may be sub-optimal. Practically, their algorithm removes suspected workers either permanently (a decision of high risk), or resets the list of suspects at each window boundary (which is sensitive to the choice of hyperparameters). Having said that, \cite{allenzhu2021byzantineresilient} prove convergence to a local minimum instead of to a saddle point as we do here. Finally, in another independent work \citet{elmhamdi2021distributed} empirically observe that using momentum may be beneficial, though they provide no theoretical guarantees.

\paragraph{Other concerns.} To deploy robust learning for real world applications, many other issues such as data heterogeneity become important~\citep{kairouz2019federated,karimireddy2020scaffold}. Robust learning algorithms which assume worker data are i.i.d. may fail in the federated learning setting~\citep{he2020byzantinerobust}.
Numerous variations have been proposed which can handle non-iid data with varying degrees of success~\citep{li2019rsa,ghosh2019robust,chen2019distributed,peng2020byzantinerobust,data2020byzantineresilienta,he2020byzantinerobust,elmhamdi2020collaborative,dong2020communicationefficient}.
Further, combining robustness with notions of privacy and security is also a crucial and challenging problem~\citep{he2020secure,so2020byzantineresilient,so2020turboaggregate,jin2020stochasticsign}. Such heterogeneity is especially challenging and can lead to backdoor attacks (which are orthogonal to the training attacks discussed here) \citep{bagdasaryan2019backdoor,sun2019can, wang2020attack} and remains an open challenge.


\section{Brittleness of existing aggregation rules}\label{subsec:aggregator-failure}
In this section, we study the robustness of existing popular Byzantine aggregation rules.
Unfortunately, we come to a surprising conclusion---most state of the art aggregators require strong non-realistic restrictions on the noise distribution. We show this frequently does not hold in practice, and present counter-examples where these aggregators fail even in the complete \emph{absence} of Byzantine workers.
State of the art aggregators such as Krum~\citep{blanchard2017machine}, coordinate-wise median (CW) ~\citep{yin2018byzantinerobust},

RFA~\citep{pillutla2019robust}, Bulyan~\citep{mhamdi2018hidden}, etc. all generalize the scalar notion of the median to higher dimensions and are hence exhibit different ways of `middle-seeking'.
At a high level, these schemes require the noise distribution to be unimodal and highly concentrated, discarding any gradients from the tail of the distribution too aggressively as `outliers'.
We give a brief summary of these rules below. We use $[\vv]_j$ to indicate the $j$th coordinate of vector $\vv$.

\textbf{Coordinate-wise median}: \vspace{-3mm}
\[
  [\text{CM}(\xx_1, \dots, \xx_n)]_j = \text{median}([\xx_1]_j, \dots, [\xx_n]_j)\,.\vspace{-0mm}
\]
\textbf{RFA} (robust federated averaging) aka geometric median: \vspace{0mm}
\[ \text{RFA}(\xx_1, \dots, \xx_n) = \argmin_{\vv} \sum_{i=1}^n\norm{\vv - \xx_i}_2\,.
\]
\textbf{Trimmed Mean}: For each coordinate $j$, compute sorting $\Pi_j$ which sorts the coordinate values. Compute the average after excluding (`trimming') $\delta n$ largest and smallest values. \vspace{-3mm}
\[ [\text{TM}(\xx_1, \dots, \xx_n)]_j = \frac{1}{n - 2 \delta n} \sum_{i = \delta n}^{n - \delta n} [\xx_{\Pi_j(i)}]_j\,.\]
\textbf{Krum}: Krum tries to select a point $\xx_i$ which is closest to the mean after excluding $\delta n + 2$ furthest away points. Suppose that $\cS \subset [n]$ of size at least $(n - \delta n - 2)$. Then, \vspace{-3mm}
\[\text{Krum}(\xx_1, \dots, \xx_n) = \argmin_{\xx_i} \min_{\cS} \sum_{j \in \cS}\norm{\xx_i - \xx_j}^2_2\,.\]

\begin{figure}[t]
  \vspace{-3mm}
  \centering
  \includegraphics[
    width=0.75\linewidth
  ]{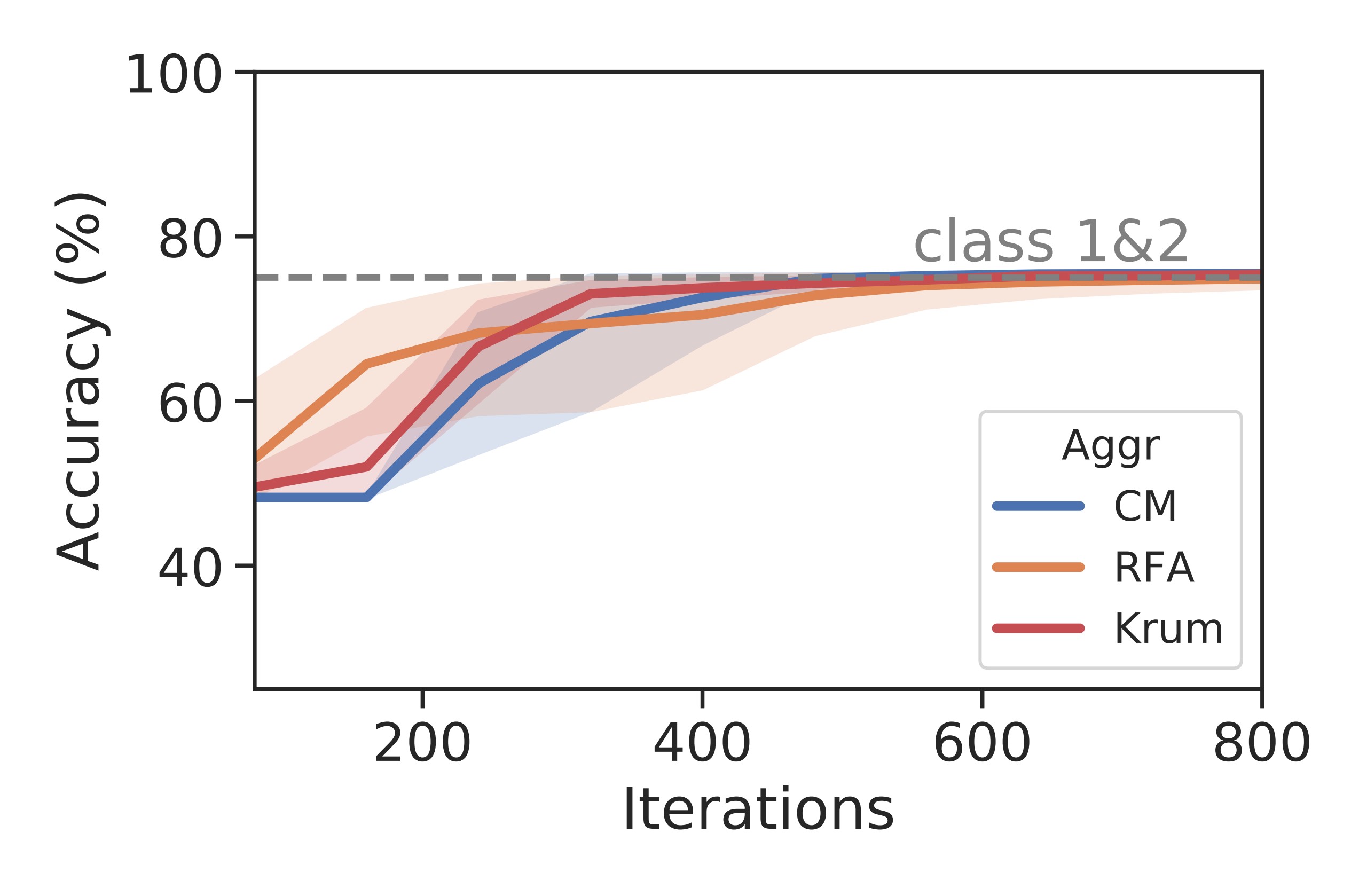}\vspace*{-5mm}
  \caption{Failure of existing methods on imbalanced MNIST dataset. Only the head classes (class 1 and 2 here) are learnt, and the rest 8 classes are ignored.
    See Sec. \ref{sub-sec:exp:long-tail}.}
  \label{fig:exp1-demo}\vspace*{-5mm}
\end{figure}

\paragraph{Counterexample 1.} Let us pick $n$ random variables $\pm 1$ with uniform probability for some odd $n$. These variables have mean 0. Since $n$ is odd, Krum, CW, Bulyan all will necessarily return either of $\pm 1$. This remains true even if we have infinite samples (large $n$), and if there are no corruptions. This simple examples illustrates the fragility of such `middle-seekers' to bimodal noise.

\paragraph{Counterexample 2.} Fig.~\ref{fig:exp1-demo} illustrates a more realistic example where imbalanced MNIST dataset causes a similar problem. Here, 0.5 fraction of data corresponds to class 1, 0.25 to class 2, and so on. The gradients over data of the same class are much closer than those of a different class. Hence, when we pick $n$ i.i.d. gradients, most them will belong to class 1 or 2 with very few belonging to the rest. Thus, coordinate-wise median, geometric median and Krum always select the gradient corresponding to classes 1 or 2, ensuring that we only optimize over these classes ignoring the rest.

\paragraph{Counterexample 3.}
Middle-seekers can also fail on continuous uni-modal distributions. Consider,\vspace*{-5mm}
\begin{wrapfigure}{r}{0.4\columnwidth}
  \begin{center}
    \includegraphics[width=0.4\columnwidth]{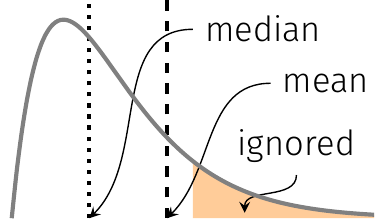}
  \end{center}
  \caption{For fat-tailed distributions, median based aggregators ignore the tail. This bias remains even if we have infinite samples.}
  \label{fig:counter}\vspace*{-5mm}
\end{wrapfigure}
\[
  p(x) = \begin{cases}
    3x^{-4} & \text{ for } x \geq 1\, \\
    0       & \text{ o.w.}
  \end{cases}
\]
This power-law distribution has mean $1.5$ and variance $0.75$. However, since the distribution is skewed, its median is $2^{1/3} \approx 1.26$ and is smaller than the mean. This difference persists even with \emph{infinite} samples showing that with imbalanced (i.e. skewed) distributions, coordinate-wise median, geometric median and Krum do not obtain the true optimum. Empirical evidence suggests that such heavy-tailed distributions abound in deep learning, making this setting very relevant to practice \citep{zhang2019adam}.

\begin{theorem}[Failure of `middle-seekers']\label{thm:failure-middle}
  There exist simple convex stochastic optimization settings with bounded variance where traditional distributed SGD converges but coordinate-wise median, RFA, and Krum do not converge to the optimum almost surely for any number of workers and even if none of them are Byzantine.
\end{theorem}

\begin{remark}[Practical usage]
  Theorem~\ref{thm:failure-middle} notes that one must be cautious while using median or Krum as aggregation rules when we suspect that our data is multi-modal (typically occurs when using small batch sizes), or if we believe our data to be heavy-tailed (typically occurs in imbalanced datasets or language tasks).
  These aggregators may suffice for standard image recognition tasks with large batch sizes since the noise is nearly Gaussian \citep{zhang2019adam}.
\end{remark}

Median based aggregators have a long and rich history in the field of robust statistics~\citep{minsker2015geometric}. However, classically the focus of robust statistics has been to design methods which can withstand a large fraction of Byzantine workers (high \emph{break down} point $\delta_{\max}$) and not result in infinities \citep{hubert2008high}. It was sufficient for the output to be bounded, but the quality of the result was not a concern. The counter examples in this section exactly stem from this issue. We will later define a finer notion of a robust statistic which accounts for both the quality of the output as well as the breakdown point $\delta_{\max}$.


\section{Necessity of using history}\label{sec:history}

\begin{figure}[t]
  \vspace{-3mm}
  \centering
  \includegraphics[
    width=0.75\linewidth
  ]{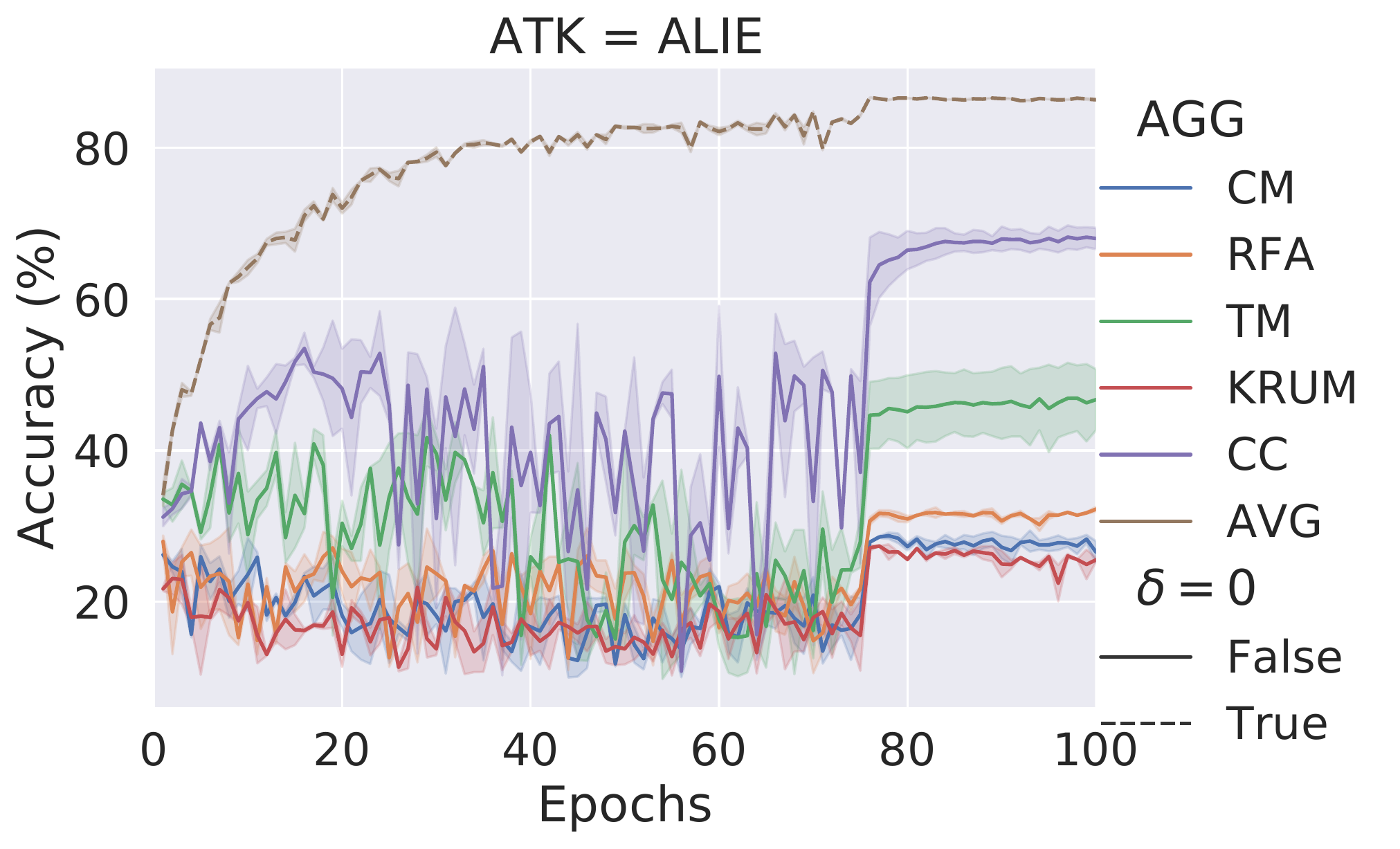}\vspace*{-5mm}
  \caption{Failure of permutation invariant algorithms on CIFAR10 dataset with \citep{baruch2019little} attack. Comparing to simple average with no attacker (dashed lines), all robust aggregators (including centered clip) see a significant drop in accuracy against time coupled attacks. See Sec. \ref{sub-sec:exp:momentum}.}
  \label{fig:history-demo}\vspace*{-5mm}
\end{figure}

Recent work \citep{baruch2019little,xie2019fall} has shown a surprising second source vulnerability for most currently popular robust aggregators. In this section we take a closer look at their attack and use our observations to make an even stronger claim---any aggregation rule which is oblivious of the past cannot converge to the optimum and retains a non-zero error even after infinite time.

The inner-product manipulation attack as defined by \cite{baruch2019little,xie2019fall} is deceptively simple.
Their attacks works by hiding small Byzantine perturbations within the variance of the good gradients. Since we only have access to noisy stochastic gradients, the aggregators fail to identify these perturbations. While this perturbation is small in any single round, these can accumulate over time.
We formalize this argument into a lower bound in Theorem~\ref{thm:robust-limits}. We show that the key reason why this attack works on algorithms such as CM, RFA, or Krum is that they are \emph{oblivious} and do not track information from previous rounds. Thus, an attacker can couple the perturbations across time eventually leading to divergence. This is also demonstrated experimentally in Fig.~\ref{fig:history-demo}.
\begin{definition}[Permutation invariant algorithm]\label{asm:perm-inv-algorithm}
  Suppose we are given an instance of $\delta$-robust optimization problem satisfying Definition~\ref{asm:robust-opt-problem}. Define the set of stochastic gradients computed by each of the $n$ workers at some round $t$ to be $[\tilde\gg_{1,t}, \dots, \tilde\gg_{n,t}]$. For a good worker $i \in \cG$, these represent the true stochastic gradients whereas for a bad worker $j \in \cB$, these represent arbitrary vectors. The output of any optimization algorithm $\Alg$ is a function of these gradients. A permutation-invariant algorithm is one which for any set of permutations over $t$ rounds $\{\pi_1,\dots,\pi_t\}$, its output remains unchanged if we permute the gradients.
  \begin{equation*}
    \Alg \rbr*{\!\begin{aligned}
         & [\tilde\gg_{1,1},..., \tilde\gg_{n,1}], \\
         & \hspace{9mm}...                         \\
         & [\tilde\gg_{1,t},..., \tilde\gg_{n,t}]
      \end{aligned}} =
    \Alg \rbr*{\!\begin{aligned}
         & [\tilde\gg_{\pi_1(1),1},..., \tilde\gg_{\pi_1(n),1}], \\
         & \hspace{13mm}...                                      \\
         & [\tilde\gg_{\pi_t(1),t},..., \tilde\gg_{\pi_t(n),t}]
      \end{aligned}}
  \end{equation*}
\end{definition}

\begin{remark}[Memoryless methods are permutation invariant]
  Any algorithm which is `memoryless' i.e. uses only the computations resulting from current round is necessarily permutation-invariant since the indices corresponding to the stochastic gradient are meaningless. It is only when these stochastic gradients are tracked over multiple rounds (i.e. we use memory) do the indices carry information.
\end{remark}
\begin{theorem}[Failure of permutation-invariant methods]\label{thm:memory-failure}
  Suppose we are given any permutation invariant algorithm $\cA$ as in Definition~\ref{asm:perm-inv-algorithm}, $\mu \geq 0$, $\delta \in [0,1]$, and $n$ large enough that $\delta n \geq 4(1 + \log t)$. Then, there exists a $\delta$-robust $\mu$ strongly-convex optimization problem satisfying Definition~\ref{asm:robust-opt-problem}, such that the output $\tilde\xx_t$ of $\Alg$ after $t$ rounds necessarily has error
  \[
    \E[f(\tilde\xx_t)] - f(\xx^\star) \geq \Omega\rbr*{ \frac{\delta \sigma^2}{\mu}}\,.
  \]
\end{theorem}
Nearly all currently popular aggregation rules, including coordinate-wise median, trimmed mean \citep{yin2018byzantinerobust}, Krum \citep{blanchard2017machine}, Bulyan \citep{mhamdi2018hidden}, RFA, geometric median \citep{ghosh2019robust}, etc. are permutation invariant and satisfy Definition~\ref{asm:perm-inv-algorithm}. Theorem~\ref{thm:memory-failure} proves a very startling result---all of them \emph{fail to converge} to the optimum even for strongly-convex problems.
Further, as $\mu$ decreases (the problem becomes less strongly-convex), the error becomes unbounded.

\begin{remark}[Fixed Byzantine workers]
  The failure of permutation-invariant algorithms also illustrates the importance of assuming that the indices of Byzantine workers are fixed across rounds. If a different fraction of workers are allowed to be Byzantine each round, then the lower bound in Theorem~\ref{thm:memory-failure} applies to all algorithms and convergence is impossible. While it is indeed a valid concern that Byzantine workers may pretend to be someone else (or more generally perform Sybil attacks where they pretend to be multiple workers), simple mechanisms such as pre-registering all participants (perhaps using some identification) can circumvent such attacks.
\end{remark}

There are very few methods which are not permutation invariant and are not subject to our lower bound. Examples include Byzantine SGD \citep{alistarh2018byzantine} which only works for convex problems, and some heuristic scoring rules such as \citep{regatti2020befriending}. There has also been a recent independent work \citep{allenzhu2021byzantineresilient} which utilizes history, but they have strong requirements on the noise (see Section~\ref{subsec:aggregator-failure} for why this might be an issue) and are not compatible with our problem setting. See Appendix~\ref{subsec:app-allenzhu-comp} for a more detailed comparison.


\section{\emph{Robust} robust aggregation}\label{sec:robust-oracle}
\arxiv{
  \begin{quotation}
    \emph{``A problem well put is half solved.''} \quad --Charles Kettering.
  \end{quotation}
}
Past work on Byzantine robust methods have had wildly varying assumptions making an unified comparison difficult. Perhaps more importantly, this lead to unanticipated failures as we saw in Sec.~\ref{subsec:aggregator-failure}. In this section, we attempt to provide a standardized specification for an robust aggregator which we believe captures a wide variety of real world behavior i.e. a robust aggregator which is robust to its assumptions. We then design a simple and efficient clipping based aggregator which satisfies this notion.


\subsection{Anatomy of a robust aggregator}\label{subsec:aggregator-properties}
Suppose that we are given an aggregation rule $\cA(\, \cdots \,)$ and $n$ vectors $\{\xx_1, \dots, \xx_n\}$. Among the given $n$ vectors, let $\cG \subseteq [n]$ be \emph{good} (i.e. satisfy some closeness property), and the rest are Byzantine (and hence can be arbitrary).
The ideal aggregator would return $\frac{1}{\abs{\cG}}\sum_{j\in \cG} \xx_{j}$ but this requires exactly identifying the good workers, and hence may not be possible. We will instead be satisfied if our aggregation rule approximates the ideal update up to some error.

Our notion of a robust aggregator is characterized by two quantities: $\delta_{\max}$ which denotes the breakdown point, and a constant $c$ which determines the quality of the solution. We want an aggregator which has as large $\delta_{\max}$ and a small $c$.
\begin{definition}[($\delta_{\max}, c$)-robust aggregator]\label{asm:agg-ass}
  Suppose that for some $\delta \leq \delta_{\max} \leq 0.5$ we are given $n$ random vectors $\xx_1\,, \dots, \xx_n$ such that a good subset $\cG \subseteq [n]$ of size at least $\abs{\cG} > (1-\delta)n$  are independent with distance bounded as \vspace*{-1mm}
  \begin{align*}
    \E\norm{\xx_i - \xx_j}^2 \leq \rho^2\,, \vspace*{-5mm}
  \end{align*}
  for any fixed $i, j \in \cG$. Then, define $\bar\xx := \frac{1}{\abs{\cG}} \sum_{j \in \cG} \xx_j$. The, the robust aggregation rule $\cA(\xx_1\,, \dots, \xx_n)$ outputs $\hat\xx$ such that,\vspace*{-2mm}
  \[ \E\norm{\hat\xx - \bar\xx}^2 \leq c\delta  \rho^2\,,
    \vspace*{-1mm}
  \]
  where the expectation is over the random variables $\{\xx_i\}_{i \in [n]}$ and randomness in the aggregation rule $\cA$.
\end{definition}
The error in Definition~\ref{asm:agg-ass} is of the order $\delta \rho^2$. Thus, if $\delta = 0$ (no Byzantine workers), we recover the ideal average of the workers exactly. Further, we recover the exact average $\bar\xx$ if $\rho=0$ (no variance) since in this case all the good points are identical and are trivial to identify if they are in the majority ($\delta \leq \delta_{\max} \leq 0.5$). We demand that when the fraction of Byzantine workers is less than the breakdown point $\delta_{\max}$, the error of the output degrades gracefully with $\delta$.

However, the error remains positive ($\delta\rho^2$) even with infinite $n$ and seems to indicate that having additional workers may not help.
It turns out that this is unfortunately the price to pay for not knowing the good subset and is unavoidable. The following theorem is adapted from standard robust estimation lower bounds (e.g. see \citet{lai2016agnostic}).
\begin{theorem}[Limits of robustness]\label{thm:robust-limits}
  There exist a set of $n$ random vectors $\xx_1\,, \dots, \xx_n$ such that a good subset $\cG \subseteq [n]$ of size at least $\abs{\cG} \geq (1-\delta)n$ is i.i.d. satisfying
  $ \E\norm{\xx_i - \xx_j}^2 \leq \rho^2 \,,$ for any apriori fixed $i, j \in \cG\,.$
  For these vectors, \emph{any} aggregation rule $\hat\xx = \cA(\xx_1\,, \dots, \xx_n)$ necessarily has an error\vspace{-3mm}
  \[
    \E\norm{\hat\xx - \bmu}^2 \geq \delta \rho^2 \,.\vspace{-3mm}
  \]
  Further, the error can be unbounded ($\infty$) if $\delta \geq \frac{1}{2}$.
\end{theorem}
This establishes Definition~\ref{asm:agg-ass} as the tightest notion of a robust aggregation oracle possible.

\begin{algorithm}[t]
  \caption{$\cA$ - Centered Clipping}\label{alg:clipping}
  \begin{algorithmic}[1]
    \STATE \textbf{input:} $(\mm_1, \dots, \mm_n)$, $\tau$, $\vv$, $L$
    \STATE  \textbf{default:} $L = 1$ and $\vv = \hat\mm$ (previous round aggreg.)
    \FOR{each iteration $l=1,\dots, L$}
    \STATE $\cc_i \leftarrow (\mm_i - \vv) \min\rbr*{1, \frac{\tau}{\norm{\mm_i - \vv}}}$
    \STATE $\vv \leftarrow \vv + \frac{1}{n} \sum_{i \in [n]}\cc_i$
    \ENDFOR
    \STATE \textbf{output:} $\vv$
  \end{algorithmic}
\end{algorithm}
\subsection{Robust aggregation via centered clipping}\label{subsec:iter-clipping}
Given that most existing aggregation rules fail to satisfy Definition~\ref{asm:agg-ass}, one may wonder if any such rule exists. We propose the following iterative \emph{centered clipping} (CC) rule: starting from some point $\vv_0$, for $l \geq 0$ compute
\begin{equation} \tag{\sc CC} \label{eqn:iter-clip}
  \vv_{l+1} = \vv_l + \frac{1}{n} \sum_{i=1}^n (\xx_i - \vv_l) \min\rbr[\big]{1, \frac{\tau_l}{\norm{\xx_i - \vv_l}}}\vspace*{-3mm}
\end{equation}
\begin{remark}[Ease of implementation]
  The centered clipping update is extremely simple to implement requiring $\cO(n)$ computation and communication per step similar to coordinate-wise median. This is unlike more complicated mechanisms such as Krum or Bulyan which require $\cO(n^2)$ computation and are hence less scalable. Further, as we will see later empirically, a single iteration of \ref{eqn:iter-clip} is often sufficient in practice. This means that the update can be implemented in an asynchronous manner \citep{chen2016revisiting}, and is compatible with secure aggregation for federated learning \citep{bonawitz2017practical}.
\end{remark}
We can formalize the convergence of this procedure.
\begin{theorem}[Robustness of centered clipping]\label{thm:iter-clip}
  Suppose that for $\delta \leq 0.1$ we are given $n$ random vectors $\xx_1\,, \dots, \xx_n$ such that a good subset $\cG \subseteq [n]$ of size at least $\abs{\cG} \geq (1-\delta)n$  are i.i.d. with variance bounded as $\E\norm{\xx_i - \xx_j}^2 \leq \rho^2$ for any fixed $i, j \in \cG$. Then, starting from any $\vv_0$ the output of centered clipping after $l$ steps $\vv_l$ satisfies
  \[
    \E\norm{\vv_l - \bar\xx}^2 \leq (9.7 \delta)^{l}3\E\norm{\vv_0 - \bar\xx}^2 + 4000 \delta \rho^2\,.
  \]
\end{theorem}
\begin{proof}[Proof Sketch]
  Suppose that we are given $\{\xx_1, \dots, \xx_n\}$ with a subset of size at most $\delta n$ are bad (denoted by $\cB$), and the rest are good ($\cG$). Consider the following simple scenario where $\norm{\xx_i}^2 \leq \rho^2$ almost surely for any $i \in \cG$. In such a case, a very simple aggregation rule exists: clip all values to a radius $\rho$ and then compute the average. All the good vectors remain unchanged. The magnitude of a clipped bad vector is at most $\rho$ and since only a $\delta$ of the vectors are bad, they can move the center by at most $\rho \delta$ ensuring that our error is $\delta^2\rho^2$. This is even better than Definition~\ref{asm:agg-ass}, which only requires the error to be smaller than $\delta \rho^2$.
  Of course there were two aspects which over-simplified our computations in the above discussion: i) we measure the pair-wise distance $\norm{\xx_i - \xx_j}$ between good workers instead of absolute norms, and ii) we do not have an almost sure bound, but only in expectation.
\end{proof}
\begin{corollary}Starting from any $\vv_0$ with an initial error estimate of $\E\norm{\vv_0 - \bar\xx}^2 \leq B^2$, running \ref{eqn:iter-clip} for $l = 100 \log\rbr*{\nicefrac{3B^2}{\delta \rho^2}}$ is a $(\delta_{\max}, c)$-robust aggregator as per Definition~\ref{asm:agg-ass} with $c = 4000$ and $\delta_{\max} = 0.1$. \\
  Further, if $\E\norm{\vv_0 - \bar\xx}^2 \leq \rho^2$ then a \emph{single} step of \ref{eqn:iter-clip} is a $(\delta_{\max}, c)$-robust aggregator.
\end{corollary}
The above corollary proves that starting from \emph{any} point $\vv_0$ and running enough iterations of \ref{eqn:iter-clip} is guaranteed to provide a robust estimate. However, if we have a good starting point, we can prove a much stronger statement---that a \emph{single} clipping step is sufficient to provide robustness. We will use this latter part in designing an efficient robust optimization scheme in the next section.

Note that we have not tried to optimize for the constants in the theorem above---there is room for improvement in bringing $\delta_{\max}$ closer to 0.5, as well as in reducing the value of $c$. This may need a more careful analysis, or perhaps even a new oracle. We leave such improvements for future.

With this, we have addressed the first stumbling block and now have a robust aggregator. Next, we see how using momentum can defend against time-coupled attacks.


\section{Robust optimization using momentum}
In this section we will show that any Byzantine robust aggregator satisfying Definition~\ref{asm:agg-ass} can be combined with (local) worker momentum, to obtain a Byzantine robust optimization algorithm which successfully defends against time coupled attacks. Every time step $t \geq 1$, the server sends the workers parameters $\xx_{t-1}$ and each good worker $i \in \cG$ sends back $\mm_{t, i}$ computed recursively as below starting from $\mm_{0,i} = 0$
\begin{equation}\label{eqn:byz-sgdm-worker}\tag{\sc worker}
  \mm_{t, i} = (1- \beta_t) \gg_i(\xx_{t-1}) +  \beta_t \mm_{t-1, i}\,.
\end{equation}
The workers communicate their momentum vector to the server instead of the stochastic gradients directly since they have a much smaller variance.
Byzantine workers may send arbitrary vectors to the server. The server then uses a Byzantine-resilient aggregation rule $\cA$ such as \eqref{eqn:iter-clip} and computes the update
\begin{equation}\label{eqn:byz-sgdm-server}\tag{\sc server}
  \begin{split}
    \mm_t &= \cA(\mm_{t,1}\,, \dots,\, \mm_{t,n})\\
    \xx_t &= \xx_{t-1} - \eta_t\mm_t \,.
  \end{split}
\end{equation}
Intuitively, using momentum with $\beta = (1- \alpha)$ averages the stochastic gradients of the workers over their past $\nicefrac{1}{\alpha}$ gradients. This results in a reduction of the variance of the good workers by a factor $\alpha$ since their noise is uncoupled. However, the variance of the time-coupled Byzantine perturbations does not reduce and becomes easy to detect.
\begin{algorithm}[!t]
  \caption{Robustness using Momentum}\label{alg:momentum}
  \begin{algorithmic}[1]
    \STATE \textbf{input:} $\xx$, $\eta$, $\beta$, $\cA$
    \STATE \textbf{initialize:} $\mm_i \leftarrow \0$ $\forall i \in [n]$
    \FOR{each round $t=1,\dots$}
    \STATE server communicates $\xx$ to workers
    \ONCLIENT{$i \in \cG$}
    \STATE compute mini-batch gradient $\gg_i(\xx)$
    \STATE compute $\mm_i \leftarrow (1-\beta)\gg_i(\xx) + \beta \mm_i$
    \STATE communicate $\mm_i$ to server
    \ENDON
    \STATE aggregate $\hat\mm = \cA(\mm_1, \dots, \mm_n)$
    \STATE update $\xx \leftarrow \xx - \eta\hat\mm$
    \ENDFOR
  \end{algorithmic}
\end{algorithm}
\subsection{Rate of convergence}
Now we prove a rate of convergence of our Byzantine aggregation algorithm.
\begin{theorem}[Byzantine robust SGDm]\label{thm:byz-sgdm-convergence}
  Suppose that we are given a $\delta$-robust problem satisfying Def.~\ref{asm:robust-opt-problem} and a $(\delta_{\max}, c)$-robust aggregation rule satisfying Def.~\ref{asm:agg-ass} for $\delta_{\max} \geq \delta$. Then, running \ref{eqn:byz-sgdm-worker} update with step-sizes $\eta_t = \min\rbr[\Big]{\sqrt{\frac{(f(\xx_0) - f^\star) + \tfrac{5c\delta}{16 L}\sigma^2}{20 L T\sigma^2 \rbr*{\tfrac{2}{n} + c\delta}}}, \frac{1}{8L}}$ and momentum parameter $\alpha_1 = 1$ and $\alpha_t = 8 L \eta_{t-1}$ for $t \geq 2$ satisfies
  \begin{align*}
    \frac{1}{T}\sum_{t=1}^T \E \norm{\nabla f(\xx_{t-1})}^2 & \leq                                                                                                                   \\
                                                            & \hspace*{-2cm} 16 \sqrt{\frac{\sigma^2 \rbr*{1 + c\delta n}}{nT} \rbr*{10L(f(\xx_0) - f^\star) + 3c\delta\sigma^2} } + \\
                                                            & \hspace*{-2cm} \frac{32L(f(\xx_{0}) - f^\star)}{T} + \frac{20 \sigma^2 (1 + c\delta n)}{nT}\,.
  \end{align*}
\end{theorem}
\begin{remark}[Convergence rate]
  The rate of convergence in Theorem~\ref{thm:byz-sgdm-convergence} is asymptotically (ignoring constants and higher order terms) of the order:
  \[
    \frac{1}{T}\sum_{t=1}^T \E \norm{\nabla f(\xx_{t-1})}^2 \lesssim \sqrt{\frac{\sigma^2}{T}\rbr[\Big]{\frac{1}{n} + \delta}} \,.
  \]
  First note that when $\delta =0$ i.e. when there are no Byzantine adversaries, we recover the optimal rate of $\frac{\sigma}{\sqrt{nT}}$ which linearly scales with the number of workers $n$. In the presence of a $\delta$ fraction of adversaries, the rate has two terms: the first term $\frac{\sigma}{\sqrt{nT}}$ which linearly scales with the number of workers $n$, and a second $\frac{\sigma\sqrt{\delta}}{\sqrt{T}}$ which depends on the fraction of adversaries $\delta$ but does not improve with increasing workers. Similar phenomenon occurs in the classical robust mean estimation setting \citep{lai2016agnostic} and is unfortunately not possible to improve.
\end{remark}

Our algorithm uses step-size $\eta$ and momentum parameter $\alpha = (1-\beta)$ of the order of $\sqrt{\frac{1}{nT\sigma^2} + \frac{\delta}{T\sigma^2}}$. Here $\delta$ represents the fraction of adversarial workers.
When there are very few bad workers with $\delta = \cO(\frac{1}{n})$, the momentum and the step-size parameters can remain as in the non-Byzantine case. As the number of adversaries increases, $\delta$ increases meaning we should use smaller learning rate and larger momentum.
Either when using linear scaling \citep{goyal2017accurate} or square-root scaling \citep{hoffer2017train}, we need to scale both the learning-rate and momentum parameters as $\rbr*{\frac{1}{n} + \delta}$ instead of the traditional $\frac{1}{n}$ in the presence of a $\delta$ fraction of adversaries.

The above algorithm and convergence analysis crucially relied on the low variance of the update from the workers using worker momentum. The very high momentum ensures that the variance of the updates from the workers to the server have a variance of the order $\sqrt{\frac{\sigma^2}{nT} + \frac{\delta \sigma^2}{T}}$. Note that this variance asymptotically goes to 0 with $T$ and is significantly smaller than the variance of the stochastic gradient~$\sigma^2$. This way, the Byzantine adversaries have very little lee-way to fool the aggregator.

\begin{figure*}[!t]
  \vspace{-3mm}
  \centering
  \includegraphics[
    width=\linewidth
  ]{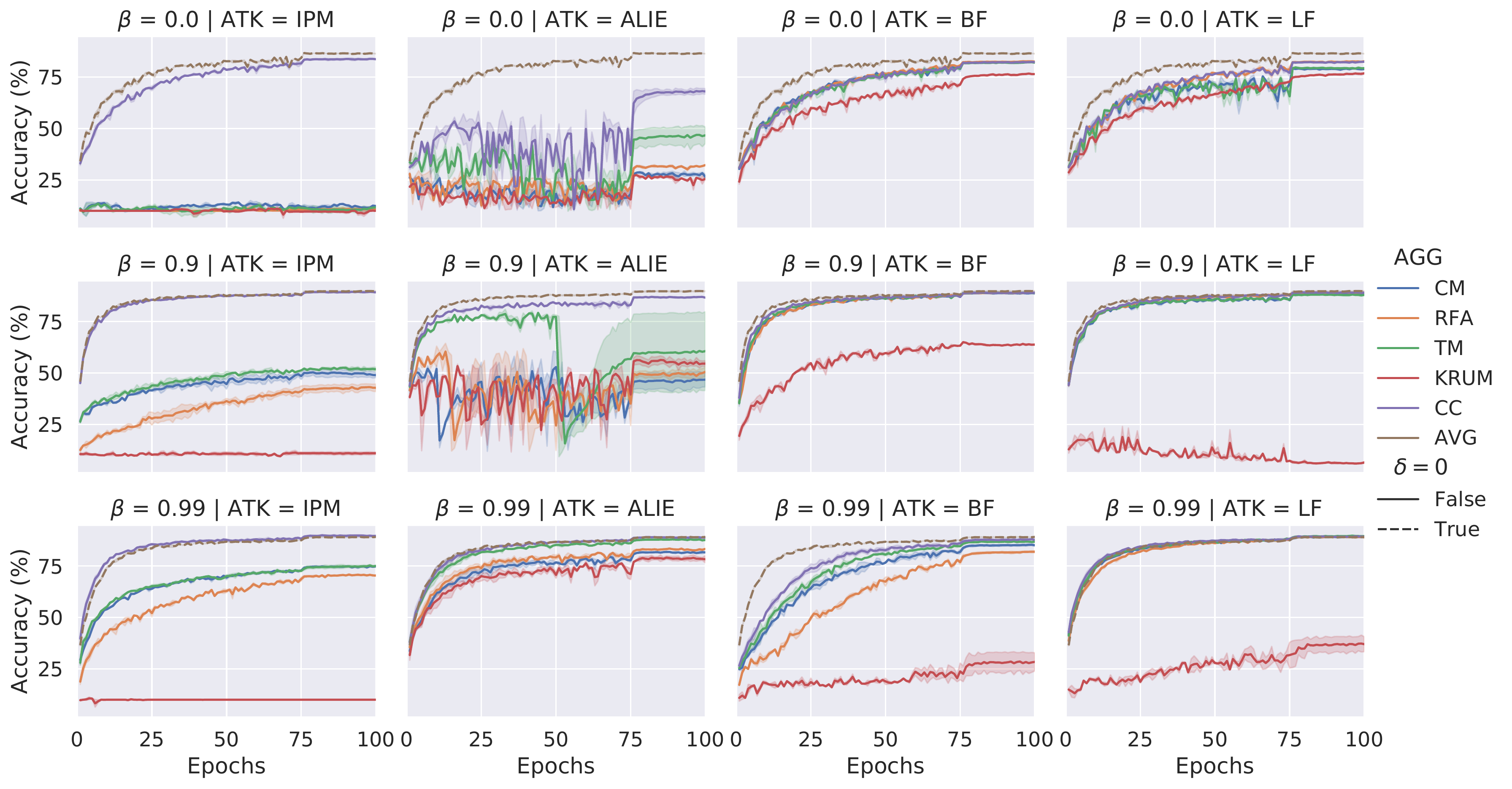}
  \caption{Coordinate median (CM), Robust Federated Aggregation (RFA), Trimmed Mean (TM), Krum, and Centered Clip (CC) are tested on Cifar10 with 25 workers. Attackers run inner-product manipulation attack (IPM) \cite{xie2019fall}, ``a little is enough'' (ALIE) \cite{baruch2019little}, bit-flipping  (BF), and label-flipping (LF). IPM uses 11 Byzantine workers while others use 5. The dashed brown line is average aggregator under no attacks ($\delta=0$).
    Momentum generally improves all methods, with larger momentum adding stability. Centered Clip (CC) consistently has the best performance.}
  \label{fig:exp3}
\end{figure*}

\begin{figure}[!t]
  \centering
  \includegraphics[
    width=\linewidth
  ]{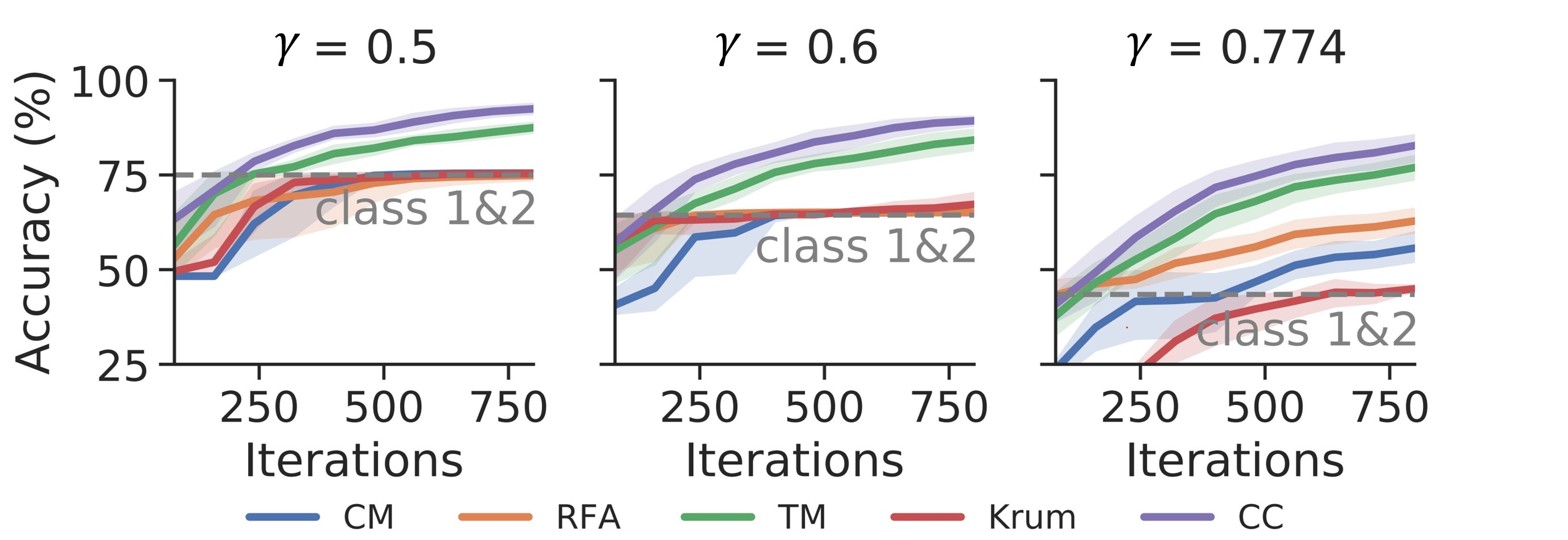}
  \caption{Robust aggregation rules on imbalanced MNIST where each successive class is a $\gamma$-fraction of the previous. Centered Clip is unaffected by imbalance where as the accuracy RFA, Krum, and CM corresponds to only learning class 1 and 2 (marked by horizontal gray dashed line).}
  \label{fig:exp1}
  \vspace{-3mm}
\end{figure}

\begin{figure}[!t]
  \vspace{-1mm}
  \centering
  \includegraphics[
    width=1\linewidth
  ]{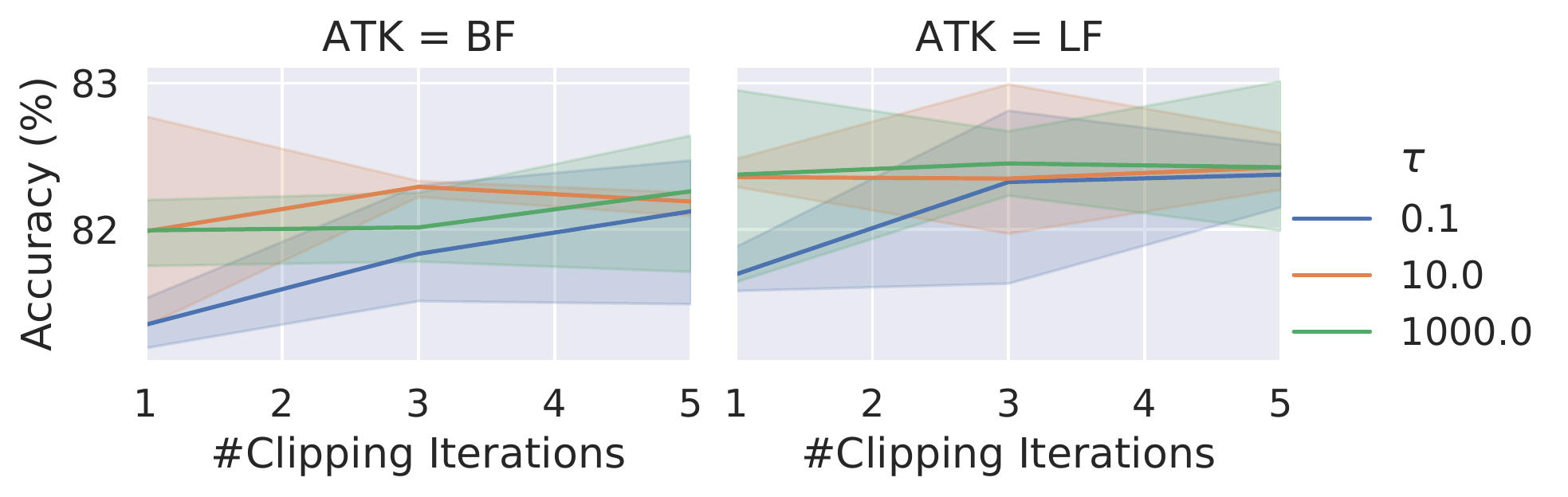}
  \caption{Final test accuracy of Centered Clip as we vary clipping iterations ($l$) and radius ($\tau$). It is stable across all hyper-parameters, justifying using $l=1$ as default.}
  \vspace{-5mm}
  \label{fig:exp2_final}
\end{figure}

\subsection{Improved convergence using MVR}
Recently, a variation of the standard momentum, called momentum based variance reduction or MVR, was proposed by \citet{tran2020hybrid,cutkosky2019momentum}. They show that by adding a small correction to correct for bias, we can improve SGD's $\cO(T^{-\frac{1}{2}})$ rate of convergence to $\cO(T^{-\frac{2}{3}})$. By combining worker momentum based variance reduction with a Byzantine robust aggregator, we can obtain a faster Byzantine robust algorithm.
\begin{theorem}[Byzantine robust MVR]\label{thm:mvr-convergence-main}
  Suppose we are given a $\delta$-robust Byzantine optimization problem Def.~\ref{asm:robust-opt-problem}. Let us run the MVR algorithm combined with a $(\delta_{\max}, c)$-robust aggregation rule \cA with $\delta \leq \delta_{\max}$, step-size $\eta = \min\cO\rbr*{\sqrt[3]{\frac{f(\xx_0) - f^\star}{T}}, \frac{1}{4L}}$, and momentum parameter $\alpha = \cO(L^2 \eta^2)$. Then,
  \[
    \frac{1}{T}\sum_{t=1}^T \E \norm{\nabla f(\xx_{t-1})}^2 \lesssim  \rbr*{\frac{L \sigma\sqrt{c\delta + 1/n}}{T}}^{2/3}\,.
  \]
\end{theorem}
Note that Theorem~\ref{thm:mvr-convergence-main} provides a significant asymptotic speedup over the traditional momentum used in Theorem~\ref{thm:byz-sgdm-convergence} and matches the lower bound of \citep{arjevani2019lower} when $\delta = 0$. This result highlights the versatility of our approach and the ease with which our notion of a Byzantine oracle can be combined with any state of the art optimization methods.

\section{Experiments}

In this section, we empirically demonstrate the effectiveness of \ref{eqn:iter-clip} and \ref{eqn:sgdm} for Byzantine-robust learning. We refer to the baseline robust aggregation rules as RFA~\citep{pillutla2019robust}, coordinate-wise median (CM), trimmed mean (TM)~\citep{yin2018byzantinerobust}, and Krum~\citep{blanchard2017machine}.
The inner iteration (T) of RFA is fixed to 3 as suggested in \citep{pillutla2019robust}.
Throughout the section, we consider the distributed training for two image classification tasks, namely MNIST \citep{lecun-mnisthandwrittendigit-2010} on 16 nodes and CIFAR-10 \citep{krizhevsky2009learning} on 25 nodes. All experiments are repeated at least 2 times. The detailed setups are deferred to \Cref{ssec:setups}.

\subsection{Failure of ``middle seekers''} \label{sub-sec:exp:long-tail}
In this experiment, we demonstrate the challenge stated in Section~\ref{subsec:aggregator-failure} by comparing robust aggregation rules on imbalanced datasets without attackers. Imbalanced training and test MNIST dataset are created by sampling classes with exponential decay, that is $1, \gamma, \gamma^2, \ldots, \gamma^{K-1}$ for classes $1$ to $K$ ($\gamma \in (0,1]$). Then we shuffle the dataset and divide it equally into 16 nodes. The mini-batch for each node is 1.

  The experimental results are presented in Fig.~\ref{fig:exp1}. For drastic decay $\gamma=0.5$, the median and geometric median based rules can only achieve 75\% accuracy which is the portion of class 1 and 2 in the data. This is a practical example of how ``middle-seekers'' fail. On the other hand, centered clip \ref{eqn:iter-clip} and trimmed mean have no such bound as they incorporate the gradients from tail distributions.

  \subsection{Impact of momentum on robust aggregation rules}\label{sub-sec:exp:momentum}

  The traditional implementation of momentum slightly differs from \eqref{eqn:byz-sgdm-worker} update and uses
  \begin{equation}\label{eqn:normal-momentum}
    \mm_{t, i} = \gg_i(\xx_{t-1}) +  \beta \mm_{t-1, i}\,.
  \end{equation}
  This version is equivalent to running \eqref{eqn:byz-sgdm-worker} update with a re-scaled learning rate of $\nicefrac{\eta}{(1-\beta)}$. Further, note that our theory predicts that the clipping radius $\tau$ should be proportional to the variance of the updates which in turn depends on the momentum parameter $\beta$. We scale $\tau$ by a factor of $(1-\beta)$ if using \eqref{eqn:byz-sgdm-worker} update, and leave it constant if using update of the form \eqref{eqn:normal-momentum}.

  In this experiment, we study the the influence of momentum on robust aggregation rules against various attacks, including bit-flipping (BF), label-flipping (LF), little is enough \citep{baruch2019little}, and inner product manipulation \citep{xie2019fall}.
  We train ResNet-20 \citep{he2016deep} on CIFAR-10 for $100$ epochs on 25 workers where 5 of them are adversaries. For \citep{xie2019fall} we use 11 Byzantine workers to amplify the attack.
  The batch size per worker is set to $32$ and the learning rate is $0.1$ before $75$th epoch and $0.01$ afterwards. Note that the smaller batch size, e.g. 32, leads to larger variance among good gradients which makes the attacks in \citep{baruch2019little,xie2019fall} more challenging.

  The results are presented in Fig.~\ref{fig:exp3}. Momentum generally makes the convergence faster and better for all aggregators, especially against SOTA attacks \citep{baruch2019little,xie2019fall}.
  \ref{eqn:iter-clip} achieves best performance in almost all experiments. More specifically, it performs especially well on \citep{baruch2019little,xie2019fall} which is very close to training without attackers ($\delta=0$).

  \subsection{Stability of Centered Clip}
  \label{sub-sec:exp:hyperparameters}
  To demonstrate the impact of two hyperparameters $\tau$, $l$ of centered clip \ref{eqn:iter-clip}, we grid search $\tau$ in $[0.1,10,1000]$ and $l$ in $[1,3,5]$.
The setup is the same as in Sec.~\ref{sub-sec:exp:momentum} and momentum is 0 to exclude its effect. The final accuracies are presented in Fig.~\ref{fig:exp2_final}. Centered clipping is very stable to the choice of hyperparameters, and can achieve good accuracy even without momentum.

\section{Conclusion}
The wildly disparate assumptions made in Byzantine robust learning not only makes comparison between different results impossible, but can also mask unexpected sources of failure. In this work, we strongly advocated for providing end to end convergence guarantees under realistic assumptions. We provided well-justified notions of a Byzantine robust aggregator and formalized the Byzantine robust stochastic optimization problem. Our theoretical lens led us to a surprisingly simple yet highly effective pair of strategies: using centered clipping and worker momentum. These strategies were thoroughly tested on a variety of attacks and shown to consistently outperform all baselines. 

\paragraph*{Acknowledgment.} We thank Dan Alistarh for useful comments and Eduard Gorbunov for pointing a mistake in our earlier proof of centered clipping. This work is partly supported by a Google Focused Research Award.

{
  \bibliography{papers}
  \bibliographystyle{icml2020}
}


\newpage
\onecolumn
\part*{Appendix}
\appendix


\section{Convergence of momentum SGD}

Here we describe the convergence proof of the naive SGD with momentum algorithm. Starting from a given $\xx_0$ and with $\mm_0 = 0$, we run the following updates with a sequence of momentum parameters $\alpha_t \in [0,1]$ and step-sizes $\eta_t \geq 0$
\begin{equation}\label{eqn:sgdm}\tag{\sc SGDm}
  \begin{split}
    \mm_{t} &= \alpha_t \gg(\xx_{t-1}) + (1-\alpha_t)\mm_{t-1}\\
    \xx_{t} &= \xx_{t-1} - \eta_t \mm_t\,.
  \end{split}
\end{equation}
While there exist numerous previous analyses of SGD with momentum for smooth non-convex objectives, most of them rely on viewing the SGDm method as an approximation of an underlying SGD without momentum algorithm---see \citet{yu2019linear,liu2020improved} for recent examples of this viewpoint. Because they view momentum as approximating an SGD process, the rates proved are necessarily slower for momentum and further they can only handle constant values of $\alpha$ (i.e. the momentum parameter cannot decrease with $T$).
In this work, we take an alternate viewpoint to momentum inspired by \citep{cutkosky2019momentum,karimireddy2020mime}. We view the momentum update as a way to reduce the variance i.e. by using an exponential averaging over many independent stochastic gradients we get an estimate of the true full gradient which has much lesser variance (though higher bias). This way, our method can handle momentum parameter which is almost 1 ($\alpha \approx \frac{1}{\sigma\sqrt{T}}$). Thus the resulting update has very low variance which will later be crucial for deriving optimal robust methods.
\begin{theorem}[Convergence of SGDm]\label{thm:sgdm-convergence}
  The \ref{eqn:sgdm} algorithm with step-size $\eta_t=\min\{\tfrac{1}{4\sigma}\sqrt{\tfrac{f(\xx_0) - f^\star}{LT}}, \tfrac{1}{4L}\}$ and momentum parameter $\alpha_1 = 1$ and $\alpha_t = 4 L \eta_{t-1}$ for $t \geq 2$ satisfies
  \[
    \frac{1}{T}\sum_{t=1}^T \E \norm{\nabla f(\xx_{t-1})}^2 \leq
    80\cdot\sigma\sqrt{\frac{L(f(\xx_0) - f^\star)}{T}} +  \frac{4L(f(\xx_0) - f^\star)}{T}
  \]
\end{theorem}
First, note that the rate for momentum algorithm is of the order $\frac{\sigma}{\sqrt{T}}$ which matches the optimal rate of SGD for smooth non-convex functions~\citep{arjevani2019lower}. Further, this rate is achieved using very high momentum with both $\alpha$ (and step-sizes) of the order $\frac{1}{\sigma\sqrt{T}}$. Also, when $\sigma = 0$ i.e. in the deterministic gradient case, we recover the optimal $\frac{1}{T}$ rate (but with a constant step-size and momentum). This is intuitive since we do not need to reduce the variance in the deterministic case and so large momentum is unnecessary.

\begin{remark}[Large batch generalization]
  There is some empirical evidence that momentum is also useful when using extremely large batch sizes (i.e. nearly deterministic gradient) and helps in closing the generalization gap \cite{shallue2018measuring}. In contrast, current theory claims that gradient descent (without momentum) is already optimal for non-convex optimization \citep{arjevani2019lower}. We believe these differences occur because even if using large batches, there remains stochasticity in the gradient due to data-augmentation. Thus $\sigma >0$ in practice even when using full batches.
\end{remark}

We first prove some supporting lemmas before proving Theorem~\ref{thm:sgdm-convergence}.
\begin{lemma}\label{lem:sgdm-descent}
  For $\alpha_1 =1$ and any $\alpha_t \in [0,1]$ for $t\geq2$, and an $L$-smooth function $f$ we have that $\E_1[f(\xx_1)] \leq f(\xx_0) - \frac{\eta_1}{2}\norm{\nabla f(\xx_0)}^2 + \frac{\eta_1}{2}\sigma^2- \frac{\eta_1}{2}(1-L\eta_1)\|\mm_1\|^2$ and for $t\geq 2$
  \[
    \E_t[f(\xx_t)] \leq f(\xx_{t-1}) + \frac{\eta_t}{2} \norm{\mm_t - \nabla f(\xx_{t-1})}^2 - \frac{\eta_t}{2}\norm{\nabla f(\xx_{t-1})}^2 - \frac{\eta_t}{2}(1-L\eta_t)\|\mm_t\|^2\,.
  \]
\end{lemma}
\begin{proof}
  By the smoothness of the function $f$ and the SGDm update,
  \begin{align*}
    f(\xx_t) & \leq f(\xx_{t-1}) - \eta_t \inp{\nabla f(\xx_{t-1})}{\mm_t} + \frac{L \eta_t^2}{2} \norm{\mm_t}^2                                                                 \\
             & = f(\xx_{t-1}) + \frac{\eta_t}{2} \norm{\mm_t - \nabla f(\xx_{t-1})}^2 - \frac{\eta_t}{2}\norm{\nabla f(\xx_{t-1})}^2 - \frac{\eta_t}{2}(1-L\eta_t)\|\mm_t\|^2\,.
  \end{align*}
  Taking conditional expectation on both sides yields the second part of the lemma. The first part follows from standard descent analysis of SGD.
\end{proof}


\begin{lemma}\label{lem:sgdm-error}
  Define $\ee_{t} := \mm_t - \nabla f(\xx_{t-1})$. Then, using any momentum and step-sizes such that $1 \geq \alpha_t \geq 4 L \eta_{t-1}$ for $t \geq 2$, we have  for an $L$-smooth function $f$ that $\E \norm{\ee_{1}}^2 \leq \alpha_1 \sigma^2$ and for $t \geq 2$
  \[
    \E \norm{\ee_{t}}^2 \leq (1 - \tfrac{\alpha_t}{2})\E\norm{\ee_{t-1}}^2 + L^2 \eta_{t-1}^2 (1 - \alpha_t)(1 + \tfrac{2}{\alpha_t})\E \norm{\mm_{t-1}}^2 + \alpha_t^2 \sigma^2\,.
  \]
\end{lemma}
\begin{proof}
  Starting from the definition of $\ee_{t}$ and $\mm_{t}$,
  \begin{align*}
    \E \norm{\ee_{t}}^2 & = \E \norm{\mm_t - \nabla f(\xx_{t-1})}^2                                                                                                                                                                    \\
                        & = \E \norm{\alpha_t \gg(\xx_{t-1}) + (1-\alpha_t)\mm_{t-1} - \nabla f(\xx_{t-1})}^2                                                                                                                          \\
                        & \le (1-\alpha_t)^2\E \norm{\mm_{t-1}-\nabla f(\xx_{t-1})}^2 + \alpha_t^2 \sigma^2                                                                                                                            \\
                        & = (1 - \alpha_t)^2\E \norm{(\mm_{t-1} - \nabla f(\xx_{t-2})) + (\nabla f(\xx_{t-2}) - \nabla f(\xx_{t-1}))}^2 + \alpha_t^2\sigma^2                                                                           \\
                        & \leq (1 - \alpha_t)(1 + \tfrac{\alpha_t}{2})\E \norm{\mm_{t-1} - \nabla f(\xx_{t-2})}^2 + (1 - \alpha_t)(1 + \tfrac{2}{\alpha_t})\E \norm{\nabla f(\xx_{t-2}) - \nabla f(\xx_{t-1})}^2 + \alpha_t^2 \sigma^2 \\
                        & \leq (1 - \tfrac{\alpha_t}{2})\E\norm{\ee_{t-1}}^2 + L^2 (1 - \alpha_t)(1 + \tfrac{2}{\alpha_t})\E \norm{\xx_{t-2} - \xx_{t-1}}^2 + \alpha_t^2 \sigma^2                                                      \\
                        & \leq (1 - \tfrac{\alpha_t}{2})\E\norm{\ee_{t-1}}^2 + L^2 \eta_{t-1}^2 (1 - \alpha_t)(1 + \tfrac{2}{\alpha_t})\E \norm{\mm_{t-1}}^2 + \alpha_t^2 \sigma^2\,.
  \end{align*}
  Here the first inequality used the fact that $\gg(\xx_{t-1})$ is an unbiased and independent stochastic gradient with variance bounded by $\sigma^2$. The second inequality follows from Fano's inequality i.e. $\norm{\xx + \yy}^2 \leq (1 + a)\norm{\xx}^2 + (1 + \frac{1}{a})\norm{\yy}^2$ for any $a \geq 0$.
\end{proof}

We are now ready to prove the convergence theorem.

\paragraph{Proof of Theorem~\ref{thm:sgdm-convergence}. }
Scaling Lemma~\ref{lem:sgdm-descent} by $L$ and adding it to Lemma~\ref{lem:sgdm-error} we have for any $t \geq 2$
\begin{align*}
  \E~L f(\xx_t) + \E \norm{\ee_{t}}^2 & \leq \E~L f(\xx_{t-1}) + \frac{L\eta_t}{2}\E \norm{\ee_{t}}^2 - \frac{L\eta_t}{2}\E\norm{\nabla f(\xx_{t-1})}^2  - \frac{L\eta_t}{2}(1-L\eta_t)\|\mm_t\|^2           \\
                                      & \hspace{2cm} +(1 - \tfrac{\alpha_t}{2})\E\norm{\ee_{t-1}}^2 + L^2 \eta_{t-1}^2 (1 - \alpha_t)(1 + \tfrac{2}{\alpha_t})\E \norm{\mm_{t-1}}^2 + \alpha_t^2 \sigma^2\,.
\end{align*}
By taking $\eta_t=\eta_{t-1}=\eta$ and $1\ge\alpha_t\ge 4L\eta$
\begin{align*}
   & \underbrace{\E~L(f(\xx_t) - f^\star) + \rbr*{1 - \tfrac{L\eta_t}{2}}\E \norm{\ee_{t}}^2+\tfrac{L\eta_t}{2}(1-L\eta_t)\|\mm_t\|^2}_{=: \xi_t}  + \tfrac{L\eta_t}{2}\E\norm{\nabla f(\xx_{t-1})}^2              \\
   & \leq {\E~L(f(\xx_{t-1}) - f^\star) + \rbr*{1 - \tfrac{\alpha_t}{2}}\E \norm{\ee_{t-1}}^2 +L^2 \eta_{t-1}^2 (1 - \alpha_t)(1 + \tfrac{2}{\alpha_t})\E \norm{\mm_{t-1}}^2} + \alpha_t^2 \sigma^2\,.
  \\
   & \leq \underbrace{\E~L(f(\xx_{t-1}) - f^\star) + \rbr*{1 - \tfrac{L\eta_{t-1} }{2}}\E \norm{\ee_{t-1}}^2 +\tfrac{L\eta_{t-1}}{2}(1-L\eta_{t-1})\E \norm{\mm_{t-1}}^2}_{=: \xi_{t-1}}  + \alpha_t^2 \sigma^2\,.
\end{align*}
Note that from the first parts of Lemma~\ref{lem:sgdm-descent} and Lemma~\ref{lem:sgdm-error}, we have
\begin{align*}
  \xi_1 & \leq {\E~L(f(\xx_1) - f^\star) + \rbr*{1 - \tfrac{L\eta_1}{2}}\E \norm{\ee_{1}}^2+\tfrac{L\eta_1}{2}(1-L\eta_1)\|\mm_1\|^2} \\
        & \leq L (f(\xx_0) - f^\star) + \sigma^2 - \tfrac{L\eta_1}{2} \E \norm{\nabla f(\xx_{0})}^2\,.
\end{align*}
Summing over $t$ and again rearranging gives
\begin{align*}
  \sum_{t=1}^\ell L \eta_t \E \norm{\nabla f(\xx_{t-1})}^2 \leq L(f(\xx_0) - f^\star) + \sum_{t=1}^\ell \alpha_t^2 \sigma^2\,.
\end{align*}
By taking $\eta_t=\eta_{t-1}=\eta$ and $\alpha_t=4L\eta$, this simplifies the above inequality to
\[
  \frac{1}{T}\sum_{t=1}^T \E \norm{\nabla f(\xx_{t-1})}^2 \leq \frac{f(\xx_0) - f^\star}{\eta T} + 16L\eta\sigma^2\,.
\]
By taking $\eta=\min\{\tfrac{1}{4\sigma}\sqrt{\tfrac{f(\xx_0) - f^\star}{LT}}, \tfrac{1}{4L}\}$ we prove the theorem.
\qed

\section{Proof of Theorem~\ref{thm:memory-failure} - Failure of permutation-invariant methods}

Our proof builds two instances of a $\delta$-robust optimization problem satisfying Definition~\ref{asm:robust-opt-problem} and shows that they are indistinguishable, meaning that we make a mistake on at least one of them.

For the first problem, set $f^{(1)}(x) = \frac{\mu}{2}x^2 - G x$ with optimum at $x^\star = \frac{G}{\mu}$ for some $G$ to be defined later. It has a gradient $\nabla f^{(1)}(x) = \mu x - G$ and we set the stochastic gradient for some $\tilde\delta \in [0,1]$ to be defined later as
\begin{equation*}
  g^{(1)}(x) = \begin{cases}
    \mu x - \sigma\tilde\delta^{-1/2} & \text{ with prob. } \tilde\delta \\
    \mu x                             & \text{ o.w.}
  \end{cases}
\end{equation*}
Defining $G := \sigma \tilde\delta^{1/2}$, we have that $g^{(1)}(x)$ is an unbiased stochastic gradient. Further, its variance is bounded by $\sigma^2$ since $\E[(g^{(1)}(x) - \nabla f^{(1)}(x))^2] \leq \sigma^2$. In each round $t$, let each worker $i \in [n]$ draw an i.i.d. sample from the distribution $g^{(1)}(x)$ as their stochastic gradient. Define $C_t \in [n]$ to be the number of workers whose stochastic gradients is the first setting i.e.
\[
  C_t = \#{\cbr*{i \in [n] \text{ s.t. } g_i^{(1)}(x_t) = \mu x_t - \sigma\tilde\delta^{-1/2}}}\,.
\]

Now we define the second problem. Let $f_2(x) = \frac{\mu}{2}x^2$ with optimum at $x^\star = 0$. Define its stochastic gradient to always be $g^{(2)}(x) = \mu x$. Now, in round $t$ each worker $i \in [n]$ computes $g^{(2)}_i(x_t) = x_t$. Then, $\min(n\delta, C_t)$ Byzantine workers corrupt their gradients to instead be $g^{(2)}_j(x_t) = \mu x_t - \sigma\tilde\delta^{-1/2}$.

Note that $C_t$ is the sum $n$ independent Bernoulli trials with parameter $\tilde \delta$. Thus, we have via Chernoff's bound that for any $\gamma \geq 2$,
\[
  \Pr[C_t > (1 + \gamma)n\tilde\delta] \leq \exp\rbr*{-\frac{\gamma n\tilde\delta}{2}}\,.
\]
By picking $\gamma = \max(2, 2 (1 + \log (T))/ (n\tilde\delta))$, we have that $\Pr[C_t > (1 + \gamma)n\tilde\delta] \leq \frac{1}{2T}$. By setting $\tilde\delta = \delta/6$ and assuming that $n$ is large enough such that $4(1 + \log T) \leq \delta n$, we can simplify $(1 + \gamma)n\tilde\delta \geq \delta n$. Taking an union bound over all values of $t$, we have that
\[
  \Pr\!\big[C_t \leq n\delta \text{ for all } t \in [T]\big] \geq \frac{1}{2}\,.
\]
Thus, with probability at least 0.5, we have that the stochastic gradients in problem 1 are exactly the same (up to permutation) to problem 2. This implies that with probability $0.5$, no permutation-invariant algorithm can distinguish between the two settings, implying that we necessarily incur an error of the order of the difference between their minima
\[
  \mu \rbr*{\frac{G}{\mu}}^2 = \frac{\sigma^2 \tilde\delta}{\mu} = \frac{\sigma^2 \delta}{6\mu}\,.
\]
\qed

\section{Proof of Theorem~\ref{thm:robust-limits} (Limits of robust aggregation)}
It is easy to establish the second result since if $\delta \geq \frac{1}{2}$, it is impossible to decide which of the subsets is good. E.g. if half of the inputs are $a$ and the other are $b$, even if we know that $\rho=0$, the good workers might correspond to either the $a$ half or the $b$ half equally likely. Assuming $\delta \leq \frac{1}{2}$, define the following binomial distribution:
\[
  \cP := \begin{cases}
    \rho \delta^{-1/2} & \text{ with prob. } \delta/2 \\
    0                  & \text{ o.w.}
  \end{cases}
\]
Suppose that each $x_i$ for all $i \in [n]$ is an iid sample drawn from $\cP$. Clearly we have that $\E(x_i - x_j)^2 \leq \rho^2$. Define $B_n \in [n]$ to be the number of samples which are equal to $\rho \delta^{-1/2}$ (with the rest being 0).
Now consider a second scenario for $\{x_i\}$: the adversary sets $\min(\delta n, B_n)$ of the variables to $\rho \delta^{-1/2}$ and the rest of the good variables are $0$.

Note that $\E[B_n] = n\delta/2$ and so by Markov's inequality we have that $\Pr[B_n \leq n\delta] \geq \frac{1}{2}$. So with at least probability $1/2$, the two cases are impossible to distinguish. However in the first case, all samples are good whereas in the second case only the 0 samples are good. Hence, any output will necessarily have an error of the order of the difference between their respective $\bar\xx$s:
\[(\E_{x\sim \cP}[x] - 0)^2 = (\rho \delta^{1/2}/2)^2 = \frac{\delta \rho^2}{4}\,.\vspace{-1em}\]

\section{Proof of Theorem~\ref{thm:iter-clip}- Robustness of iterative clipping}
First, suppose that $\delta = 0$. In this case, our choice of clipping radius $\tau_l = \tilde\cO(\nicefrac{\rho}{\sqrt{\delta}}) = \infty$ means that we will simply averages all points. Hence, we recover $\bar\xx$ exactly with no error as required. Now if $\delta > 0$, this means that at least one of the $n$ workers is Byzantine and hence $\delta \geq \nicefrac{1}{n}$. We consider this case in the rest of the proof.

Recall that $\bar\xx = \frac{1}{\abs{\cG}}\sum_{i \in \cG}\xx_i$ and let us define $\muv = \E[\xx_j]$ for any fixed $j \in \cG$. Now since the good random vectors are iid, we have
\[
  \E\norm{\bar\xx - \muv}^2 \leq \frac{\rho^2}{\abs{\cG}} \leq \frac{2\rho^2}{n} \leq 2\delta \rho^2\,.
\]

We will first analyze a single step of centered clipping assuming we have access to $\vv$ such that i) $\vv$ is independent of the samples $\{\xx_i | i \in \cG\}$, and ii) $\E\norm{\vv - \bmu}^2 \leq \cO(\rho^2)$. Then, we will next see how to construct such a $\vv$. Our proof is inspired by \citep{zhang2019adam,gorbunov2020stochastic} who analyze the bias of clipping under heavy-tailed noise.

\subsection{Single iteration with good starting point}

Let us suppose that at some round $l$, we have the following properties:
\begin{itemize}
  \item We have a good estimate of the mean satisfying $\E\norm{\vv_l - \bmu}^2 \leq B_l^2$ where $B_l$ is a known deterministic constant.
  \item The starting point $\vv_l$ is statistically independent of $\{ \xx_i | i \in \cG\}$.
\end{itemize}
Define indicator variables $\ind_{i,l} := \ind\{\norm{\vv_l - \xx_i} \geq \tau_l\}$ which define the event that the vector $\xx_i$ is clipped, as well the resulting clipped vector
\[
  \yy_{i,l} := \vv_l + (\xx_i - \vv_l) \min\rbr*{1, \frac{\tau_l}{\norm{\xx_i - \vv_l}}}\,.
\]
The output can also be written in this new notation as
\[
  \vv_{l+1} = \frac{1}{n}\sum_{i \in [n]} \yy_{i, l} = (1 - \delta) \frac{1}{\abs{\cG}}\sum_{i \in \cG} \yy_{i,l} + \delta \frac{1}{\abs{\cB}}\sum_{j\in\cB} \yy_{j,l} \,.
\]
Then the error can be decomposed as
follows
\begin{align*}
  \E\norm{\vv_{l+1} - \bmu}^2 & = \E\norm*{(1 - \delta) \frac{1}{\abs{\cG}}\sum_{i \in \cG} \yy_{i,l} + \delta \frac{1}{\abs{\cB}}\sum_{j\in\cB} \yy_{j,l} - \frac{1}{\abs{\cG}}\sum_{i \in \cG}\E[\xx_i]}^2                                                                                                                                                          \\
                              & = \E\norm*{(1 - \delta) \frac{1}{\abs{\cG}}\sum_{i \in \cG} (\yy_{i,l} - \E[\xx_i]) + \delta \frac{1}{\abs{\cB}}\sum_{j\in\cB} (\yy_{j,l} - \bmu)}^2                                                                                                                                                                                  \\
                              & \leq 2(1 - \delta)^2 \E \norm*{ \frac{1}{\abs{\cG}}\sum_{i \in \cG}\yy_{i,l} - \E[\xx_i]}^2 + 2\delta^2 \frac{1}{\abs{\cB}}\sum_{j\in\cB} \E \norm*{\yy_{j,l} - \bmu}^2                                                                                                                                                               \\
                              & = 2(1 - \delta)^2 \underbrace{\norm*{ \frac{1}{\abs{\cG}}\sum_{i \in \cG}\E [\yy_{i,l}] - \E[\xx_i]}^2}_{\cT_1} + 2(1 - \delta)^2\underbrace{\E \norm*{ \frac{1}{\abs{\cG}}\sum_{i \in \cG}\yy_{i,l} - \E[\yy_{i,l}]}^2}_{\cT_2} + 2\delta^2 \underbrace{\frac{1}{\abs{\cB}}\sum_{j\in\cB} \E \norm*{\yy_{j,l} - \bmu}^2}_{\cT_3} \,.
\end{align*}
Thus, the error can be decomposed into 3 terms: $\cT_1$ corresponds to the bias introduced by our clipping operation in the good workers, $\cT_2$ is the variance of the clipped good workers, and finally $\cT_3$ is the error due to the bad workers. We will analyze each of the three errors in turn,

\paragraph{$\cT_3$.} For any bad index $j \in \cB$, we can bound the error using our clipping radius as for any parameter $\gamma > 0$ as
\[
  \E \norm*{\yy_{j,l} - \bmu}^2 \leq (1 + \tfrac{1}{\gamma})\E \norm*{\yy_{j,l} - \vv_l}^2 + (1+\gamma)\E \norm*{\vv_l - \bmu}^2 \leq (1 + {\gamma})\tau_l^2 + (1 + \tfrac{1}{\gamma})B_l^2\,.
\]
The first step used Young's inequality.
Further, the error due to the bad buys is also smaller if our initial estimation error $B_l^2$ is small.

\paragraph{$\cT_1$.} We then compute the bias in the update of a good worker $i \in \cG$ due to the clipping operation. Let $\ind_{i,l}$ be an indicator variable denoting if the $i$th worker was clipped (i.e. its distance from $\vv_l$ exceeding $\tau_l$).
Note that if $\ind_{i,l} = 0$, we have that $\yy_{i, l} = \xx_i$. Then,
\begin{align*}
  \E\norm*{\yy_{i,l} - \xx_i} & = \E \ind_{i,l}\norm*{\yy_{i,l} - \xx_i]} \leq \E \ind_{i,l}\norm*{\vv_{l} - \xx_i]} \leq \frac{\E \ind_{i,l}\norm*{\vv_{l} - \xx_i}^2}{\tau}            \\
                              & \leq \frac{\E \norm{\vv_{l} - \xx_i}^2}{\tau} \leq \frac{(1 + \tfrac{1}{\gamma})\E \norm{\vv_{l} - \bmu}^2 + (1 + {\gamma})E\norm{\xx_i - \bmu}^2}{\tau} \\
                              & \leq \frac{(1 + \tfrac{1}{\gamma})\rho^2 + (1 + {\gamma})B_l^2}{\tau}\,.
\end{align*}
Using this, we can compute the error as
\begin{align}\label{eqn:intermediate-error}
  \cT_1 & \leq \frac{1}{\abs{\cG}}\sum_{i \in \cG} \norm*{\E [\yy_{i,l}] - \E[\xx_i]}^2 \leq \frac{1}{\abs{\cG}}\sum_{i \in \cG} (\E\norm*{\yy_{i,l} - \xx_i})^2 \leq \frac{((1 + \tfrac{1}{\gamma})\rho^2 + (1 + {\gamma})B_l^2)^2}{\tau^2}\,.
\end{align}

\paragraph*{$\cT_2$}. Since $\vv_l$ is independent of $\{ \xx_i | i \in \cG\}$, the random vectors $\{\yy_i | i \in \cG\}$ are also independent of each other. We then have,
\begin{align*}
  \cT_2 & = \E \frac{1}{(\abs{\cG})^2}\sum_{i \in \cG}\norm*{ \yy_{i,l} - \E[\yy_{i,l}]}^2 \\
        & \leq \E \frac{1}{(\abs{\cG})^2}\sum_{i \in \cG}\norm*{ \xx_{i} - \E[\xx_{i}]}^2  \\
        & \leq \frac{\rho^2}{\abs{\cG}} \leq \frac{2\rho^2}{n} \leq 2\delta \rho^2\,.
\end{align*}
The equality in the first step used the fact that the quantities were independent, and the next inequality follows because of the contractivity of a clipping (projection) step. The last used the fact that $\abs{\cG} \geq \nicefrac{n}{2}$.

Combining the three error terms, we have
\begin{align*}
  \E\norm{\vv_{l+1} - \bmu}^2 & \leq 2(1-\delta)^2 \frac{((1 + \tfrac{1}{\gamma})\rho^2 + (1 + {\gamma})B_l^2)^2}{\tau^2} + 2(1-\delta)^2 2\delta \rho^2 + 2\delta^2 \rbr*{(1 + {\gamma})\tau_l^2 + (1 + \tfrac{1}{\gamma})B_l^2}
  \\ &= (4(1- \delta) \delta (1 + \gamma)^{3/2} + 2(1 + \tfrac{1}{\gamma})\delta^2)B_l^2 + 4 (1-\delta)^2 \delta \rho^2 + (4(1-\delta) (1 + \tfrac{1}{\gamma}) \sqrt{1 + \gamma}) \delta\rho^2
  \\ &\leq (4(1- \delta) \delta (1 + \tfrac{1}{3})^{3/2} + 8\delta^2)B_l^2 + 4\delta \rho^2 + (16 \sqrt{1 + \tfrac{1}{3}}) \delta\rho^2\\
                              & \leq (6.158\delta(1 - \delta) + 8\delta^2) B_l^2 + 22\delta \rho^2     \,.
\end{align*}
The last step used $\gamma = \tfrac{1}{3}$. The equality in the second step used a clipping radius of
\[
  \tau_l^2 = 4(1-\delta)\frac{(4\rho^2 + \tfrac{4}{3}B_l^2)}{\sqrt{3}\delta}\,.
\]
Thus, we have 
\begin{equation}\label{eqn:single-cc}
  \norm{\vv_{l+1} - \bmu}^2 \leq (6.158\delta(1 - \delta) + 8\delta^2)  B_l^2 + 22\delta \rho^2 \leq 8\delta B_l^2 + 22\delta\rho^2 \,.
\end{equation}
In many cases, we will have access to a good starting point satisfying $B_l^2 = \cO(\rho^2)$. For example, suppose we knew that $\E\norm{\xx_i}^2 \leq b\rho^2$ for any fixed $i \in \cG$. Then, then $B_l^2 = b\rho^2$ with $\vv_l = \0$. In such cases, the above proof shows that a single iteration of centered clipping is sufficient to give a robust aggregator.

\subsection{Robustness starting from arbitrary point}
In this section, we will give an algorithm for those cases where we do not have access to any good starting point. Then, we proceed as follows: first, we partition the given dataset $\cX = \{\xx_1, \dots, \xx_n\}$ randomly into $\cX_1$ and $\cX_2$ of sizes $\abs{\cX_1} = 2n/3$ and $\abs{\cX_2} = n/3$. Note that the fraction of Byzantine workers in each of these is at most $\abs{\cB} < \delta n = \underbrace{1.5\delta}_{=:\delta_1}\ \abs{\cX_1} = \underbrace{3\delta}_{=:\delta_2}\ \abs{\cX_2}$. Our strategy then is to compute $\vv_l$ with $\cO(\rho^2)$ error  using set $\cX_1$, and then run a single step of centered clipping using data $\cX_2$. By \eqref{eqn:single-cc}, we can guarantee that the output will have error $\cO(\delta \rho^2)$.

\paragraph{Computing a good starting point.} Starting from an arbitrary $\vv_0$ with error $\E\norm{\vv_0 - \bmu}^2 \leq B_0^2$, we will repeatedly apply centered clipping \eqref{eqn:iter-clip}. Consider iteration $l \geq 0$ with error $\E\norm{\vv_l - \bmu}^2 \leq B_l^2$. Then, to analyze the error of $\E\norm{\vv_{l+1} - \bmu}^2$, we will proceed exactly as in the single iteration case up to equation \eqref{eqn:intermediate-error}. However, while analyzing the error of $\cT_2$, we can no longer rely on $\{\yy_1, \dots, \yy_n \}$ being independent. Hence, this step instead becomes
\begin{align*}
  \cT_2 &\leq \E \frac{1}{\abs{\cG}}\sum_{i \in \cG}\norm*{ \yy_{i,l} - \E[\yy_{i,l}]}^2 \\
        & \leq \E \frac{1}{\abs{\cG}^2}\sum_{i \in \cG}\norm*{ \xx_{i} - \E[\xx_{i}]}^2  \\
        & \leq \rho^2\,.
\end{align*}
Combining the previous bounds for the errors of $\cT_1$ and $\cT_3$ with the above yields
\begin{align*}
  \E\norm{\vv_{l+1} - \bmu}^2 & \leq 2(1-\delta)^2 \frac{((1 + \tfrac{1}{\gamma})\rho^2 + (1 + {\gamma})B_l^2)^2}{\tau^2} + 2(1-\delta)^2  \rho^2 + 2\delta^2 \rbr*{(1 + {\gamma})\tau_l^2 + (1 + \tfrac{1}{\gamma})B_l^2}
  \\ &= (4(1- \delta) \delta (1 + \gamma)^{3/2} + 2(1 + \tfrac{1}{\gamma})\delta^2)B_l^2 + 2(1-\delta)^2  \rho^2 + (4(1-\delta) (1 + \tfrac{1}{\gamma}) \sqrt{1 + \gamma}) \delta\rho^2
  \\ &\leq (4(1- \delta) \delta (1 + \tfrac{1}{3})^{3/2} + 8\delta^2)B_l^2 + 2 \rho^2 + (16 \sqrt{1 + \tfrac{1}{3}}) \delta\rho^2\\
                              & \leq (6.158\delta(1 - \delta) + 8\delta^2) B_l^2 + (20\delta + 2) \rho^2 \\
                              &\leq  6.45\delta B_l^2 + 5\rho^2  \,.
\end{align*}
The last step assumed $\delta \leq 0.15$, and the step before that used $\gamma = \tfrac{1}{3}$. The equality in the second step used a clipping radius of
\[
  \tau_l^2 = 4(1-\delta)\frac{(4\rho^2 + \tfrac{4}{3}B_l^2)}{\sqrt{3}\delta}\,.
\]
Note that this holds for \emph{any iteration} $l$ and we did not make any assumptions on $\vv_l$. Hence, we can define $B_{l+1}^2 = 6.45\delta B_l^2 + 5\rho^2$. With this, we can guarantee that for any $l \geq 0$, we have $\E\norm{\vv_{l} - \bmu}^2 \leq B_{l}^2$ where
\begin{equation}\label{eqn:multi-cc}
  B_{l}^2 \leq (6.45\delta)^l B_0^2 + 154 \rho^2 \text{ for } \delta \leq 0.15\,.
\end{equation}

\paragraph*{Putting it together.} Let us run the above procedure for $l$ steps on $\cX_1$ with $\delta_1 = 1.5\delta$. Then, by \eqref{eqn:multi-cc} we can guarantee that $\vv_l$ satisfies
\[
  B_{l}^2 \leq (9.7\delta)^l B_0^2 + 154 \rho^2 \text{ for } \delta \leq 0.1\,.
\]
Since $\vv_l$ was computed only using $\cX_1$, it is independent of $\cX_2$ and hence by \eqref{eqn:single-cc} has an error with $\delta_2 = 3\delta$ for $\delta \leq 0.1$:
\begin{align*}
  \E\norm{\vv_{l+1} - \bmu}^2 &\leq 24 \delta B_l^2 + 66 \delta\rho^2 \\
      &\leq 24 \delta \rbr*{(9.7\delta)^l B_0^2 + 154 \rho^2} + 66 \delta\rho^2\\
      &\leq (9.7\delta)^{l+1}2.5 B_0^2 + 3762\delta\rho^2\,.
\end{align*}
Now, we can finish the proof of the theorem as
\begin{align*}
  \E\norm{\vv_{l+1} - \bar\xx}^2 &\leq (1 + \tfrac{1}{99})\E\norm{\vv_{l+1} - \bmu}^2 + 100\E\norm{\bmu - \bar\xx}^2\\
  &\leq (1 + \tfrac{1}{99})\E\norm{\vv_{l+1} - \bmu}^2 + 100 \delta\rho^2\\
  &\leq (9.7\delta)^{l+1}3 B_0^2 + 4000\delta\rho^2 \,.
\end{align*}
Our theory considers this two stage procedure only due to a technicality. We believe that the single stage method also yields similar guarantees, and leave its analysis (along with obtaining a better $\delta_{\max}$) for future work.

\qed


\section{Proof of Theorem~\ref{thm:byz-sgdm-convergence} - Byzantine-Robust Convergence}

We state several supporting Lemmas before proving our main Theorem~\ref{thm:byz-sgdm-convergence}.
\begin{lemma}[Aggregation error]\label{lem:sgdm-byz-agg-err}
  Given that Definition~\ref{asm:agg-ass} holds, and that we use momentum constant parameter with $\alpha_1 = 1$ and $\alpha_t = \alpha$ for $t\geq2$, the error between the ideal average momentum $\bar\mm_t$ and the output of the robust aggregation rule $\mm_t$ for any $t\geq 2$ can be bounded as
  \[
    \E\norm{\mm_t - \bar\mm_t}^2 \leq 2c\delta \sigma^2(\alpha + (1 - \alpha)^{t-1})\,.
  \]
  For $t=1$ we can simplify the bound as $\E\norm{\mm_1 - \bar\mm_1}^2 \leq  2c\delta \sigma^2$.
\end{lemma}
\begin{proof}
  Expanding the definition of the worker momentum for any two good workers $i, j \in \cG$,
  \begin{align*}
    \E\norm{\mm_{i, t} - \mm_{j, t}}^2 & = \E\norm{\alpha_t (\gg_i(\xx_{t-1}) - \gg_j(\xx_{t-1})) + (1 - \alpha_t)(\mm_{i, t-1} - \mm_{j, t-1})}^2 \\
                                       & \leq \E\norm{(1 - \alpha_t)(\mm_{i, t-1} - \mm_{j, t-1})}^2 + 2\alpha_t^2 \sigma^2                        \\
                                       & \leq (1 - \alpha_t)\E\norm{\mm_{i, t-1} - \mm_{j, t-1}}^2 + 2\alpha_t^2 \sigma^2\,.
  \end{align*}
  Recall that we use $\alpha_1 = 1$ and a fixed momentum $\alpha_t = \alpha$ the rest of the steps. Unrolling the recursion above yields
  \begin{align*}
    \E\norm{\mm_{i, t} - \mm_{j, t}}^2 \leq \rbr*{\sum_{\ell=2}^t(1 - \alpha)^{t-\ell}} 2\alpha ^2\sigma^2 + (1 - \alpha)^{t-1}2\sigma^2 \leq 2\sigma^2(\alpha + (1 - \alpha)^{t-1})\,.
  \end{align*}
  The previous computation shows that all the good vectors given to the server are close to each other with $\rho^2 = 2\sigma^2(\alpha + (1 - \alpha)^{t-1})$. Hence, by Definition~\ref{asm:agg-ass} the output of the aggregation rule $\cA(\mm_{t,1},\, \dots,\, \mm_{t,n})$ satisfies the lemma statement.
\end{proof}

\begin{lemma}[Descent bound]\label{lem:sgdm-byz-descent}
  For $\alpha_1 =1$ and any $\alpha_t \in [0,1]$ for $t\geq2$, $\eta_t \leq \frac{1}{L}$, and an $L$-smooth function $f$ we have for any $t\geq 1$
  \[
    \E_t[f(\xx_t)] \leq f(\xx_{t-1}) - \frac{\eta_t}{2}\norm{\nabla f(\xx_{t-1})}^2 + \eta_t\E_t\norm{\bar\ee_t}^2 + \eta_t\E_t\norm{\mm_t - \bar\mm_t}^2\,.
  \]
  where $\bar\ee_t:=\bar\mm_t - \nabla f(\xx_{t-1})$.
\end{lemma}
\begin{proof}
  By the smoothness of the function $f$ and the server update,
  \begin{align*}
    f(\xx_t) & \leq f(\xx_{t-1}) - \eta_t \inp{\nabla f(\xx_{t-1})}{\mm_t} + \frac{L \eta_t^2}{2} \norm{\mm_t}^2                                   \\
             & \leq f(\xx_{t-1}) - \eta_t \inp{\nabla f(\xx_{t-1})}{\mm_t} + \frac{\eta_t}{2} \norm{\mm_t}^2                                       \\
             & = f(\xx_{t-1}) + \frac{\eta_t}{2} \norm{\mm_t - \nabla f(\xx_{t-1})}^2 - \frac{\eta_t}{2}\norm{\nabla f(\xx_{t-1})}^2               \\
             & = f(\xx_{t-1}) + \frac{\eta_t}{2} \norm{\mm_t \pm \bar\mm_t - \nabla f(\xx_{t-1})}^2 - \frac{\eta_t}{2}\norm{\nabla f(\xx_{t-1})}^2 \\
             & \leq f(\xx_{t-1}) + \eta_t\norm{\bar\ee_t}^2 + \eta_t\norm{\mm_t - \bar\mm_t}^2 - \frac{\eta_t}{2}\norm{\nabla f(\xx_{t-1})}^2\,.
  \end{align*}
  Taking conditional expectation on both sides yields the second part of the lemma.
\end{proof}

\begin{lemma}[Error bound]\label{lem:sgdm-byz-error}
  Using any constant momentum and step-sizes such that $1 \geq \alpha \geq 8 L \eta$ for $t \geq 2$, we have  for an $L$-smooth function $f$ that $\E \norm{\bar\ee_{1}}^2 \leq \tfrac{2\sigma^2}{n}$ and for $t \geq 2$
  \begin{align*}
    \E \norm{\bar\ee_{t}}^2 & \leq  (1 - \tfrac{2\alpha}{5})\E\norm{\bar\ee_{t-1}}^2  + \tfrac{\alpha}{10}\E\norm{\nabla f(\xx_{t-2}) }^2  + \tfrac{\alpha}{10}\E\norm{\mm_{t-1} - \bar\mm_{t-1}}^2 + \alpha^2 \tfrac{2\sigma^2}{n}\,.
  \end{align*}
\end{lemma}
\begin{proof}
  Using the definitions \eqref{eqn:sgdm-bar-defns} and proceeding as in Lemma~\ref{lem:sgdm-error}, we have
  \begin{align*}
    \E \norm{\bar\ee_{t}}^2 & = \E \norm{\bar\mm_t - \nabla f(\xx_{t-1})}^2                                                                                                                                                                                 \\
                            & = \E \norm{\alpha_t \bar\gg(\xx_{t-1}) + (1-\alpha_t)\bar\mm_{t-1} - \nabla f(\xx_{t-1})}^2                                                                                                                                   \\
                            & \leq \E \norm{\alpha_t \nabla f(\xx_{t-1}) + (1-\alpha_t)\bar\mm_{t-1} - \nabla f(\xx_{t-1})}^2 + \alpha_t^2 \tfrac{2\sigma^2}{n}                                                                                             \\
                            & = (1 - \alpha_t)^2\E \norm{(\bar\mm_{t-1} - \nabla f(\xx_{t-2})) + (\nabla f(\xx_{t-2}) - \nabla f(\xx_{t-1}))}^2 + \alpha_t^2 \tfrac{2\sigma^2}{n}                                                                           \\
                            & \leq (1 - \alpha_t)(1 + \tfrac{\alpha_t}{2})\E \norm{(\bar\mm_{t-1} - \nabla f(\xx_{t-2}))}^2 + (1 - \alpha_t)(1 + \tfrac{2}{\alpha_t})\E\norm{\nabla f(\xx_{t-2}) - \nabla f(\xx_{t-1})}^2 + \alpha_t^2 \tfrac{2\sigma^2}{n} \\
                            & \leq (1 - \tfrac{\alpha_t}{2})\E\norm{\bar\ee_{t-1}}^2 + \tfrac{2 L^2}{\alpha_t}\E\norm{\xx_{t-2} - \xx_{t-1}}^2 + \alpha_t^2 \tfrac{2\sigma^2}{n}                                                                            \\
                            & = (1 - \tfrac{\alpha_t}{2})\E\norm{\bar\ee_{t-1}}^2 + \tfrac{2 L^2 \eta_{t-1}^2}{\alpha_t}\E\norm{\mm_{t-1}}^2 + \alpha_t^2 \tfrac{2\sigma^2}{n}\,.
  \end{align*}
  Note that we have $\frac{2\sigma^2}{n}$ instead of simply $\sigma^2$ since we average the momentums (and hence also the stochastic gradients) over all the good workers (who number at least $n/2$). Another difference is that in the last equality we have the robust aggregate $\mm_{t-1}$ instead of the average momentum $\bar\mm_{t-1}$. We can proceed as
  \begin{align*}
    \E \norm{\bar\ee_{t}}^2 & \leq (1 - \tfrac{\alpha}{2})\E\norm{\bar\ee_{t-1}}^2 + \tfrac{2 L^2 \eta^2}{\alpha}\E\norm{\mm_{t-1}}^2 + \alpha^2 \tfrac{2\sigma^2}{n}                                                                                                                                          \\
                            & =(1 - \tfrac{\alpha}{2})\E\norm{\bar\ee_{t-1}}^2 + \tfrac{2 L^2 \eta^2}{\alpha}\E\norm{\mm_{t-1} \pm \bar\mm_{t-1} \pm \nabla f(\xx_{t-2}) }^2 + \alpha^2 \tfrac{2\sigma^2}{n}                                                                                                   \\
                            & \leq (1 - \tfrac{\alpha}{2})\E\norm{\bar\ee_{t-1}}^2 + \tfrac{6 L^2 \eta^2}{\alpha}\norm{\bar\ee_{t-1}}^2 +   \tfrac{6 L^2 \eta^2}{\alpha}\E\norm{ \mm_{t-1} - \bar\mm_{t-1}}^2 + \tfrac{6 L^2 \eta^2}{\alpha}\E\norm{\nabla f(\xx_{t-2}) }^2 + \alpha^2 \tfrac{2\sigma^2}{n}\,.
  \end{align*}
  Our choice of the momentum parameter $\alpha$ implies $64L^2\eta^2 \leq \alpha^2$ and yields the lemma statement.

\end{proof}

\paragraph{Proof of Theorem~\ref{thm:byz-sgdm-convergence}.}
We will loosely follow the proof of vanilla SGDm proof in Theorem~\ref{thm:sgdm-convergence}. Recall that $\cG$ denotes the good set and $\cB$ denotes the bad Byzantine workers with $\abs{\cG} \leq (1-\delta)n$ and $\abs{\cB} = n - \abs{\cG} \leq \delta n$. Define the ideal momentum and error as
\begin{equation}\label{eqn:sgdm-bar-defns}
  \bar\mm_{t} := \frac{1}{\abs{\cG}}\sum_{j \in \cG} \mm_{t, j}\,, \quad \bar\ee_t := \bar\mm_t - \nabla f(\xx_{t-1}) \,, \quad \text{and } \bar\gg(\xx_{t-1}) = \frac{1}{\abs{\cG}}\sum_{j \in \cG}\gg_j(\xx_{t-1})\,.
\end{equation}

Now scale the modified error bound Lemma~\ref{lem:sgdm-byz-error} by $\frac{5\eta}{2 \alpha}$ and add it to the modified descent bound Lemma~\ref{lem:sgdm-byz-descent} taking expectations on both sides to get for $t \geq 2$
\begin{align*}
  \E[f(\xx_t)] + \tfrac{5\eta}{2\alpha}\E \norm{\bar\ee_{t}}^2 & \leq \E[f(\xx_{t-1})] - \tfrac{\eta}{2}\E\norm{\nabla f(\xx_{t-1})}^2 + \eta\E\norm{\bar\ee_t}^2 + \eta\E\norm{\mm_t - \bar\mm_t}^2 + \\&\hspace{1cm} \tfrac{5\eta}{2\alpha}\E\norm{\bar\ee_{t-1}}^2 - \eta\E\norm{\bar\ee_{t-1}}^2  + \tfrac{\eta}{4}\E\norm{\nabla f(\xx_{t-2}) }^2  + \tfrac{\eta}{4}\E\norm{\mm_{t-1} - \bar\mm_{t-1}}^2 +  5\eta\alpha\frac{\sigma^2}{n}
\end{align*}
Rearranging the above terms and using the bound in the aggregation error Lemma~\ref{lem:sgdm-byz-agg-err} yields the recursion
\begin{align*}
  \underbrace{\E~f(\xx_t) - f^\star + (\tfrac{5 \eta}{2\alpha} - \eta)\E \norm{\bar\ee_{t}}^2 + \frac{\eta}{4}\E\norm{\nabla f(\xx_{t-1})}^2}_{=: \xi_t} & \leq \underbrace{\E~f(\xx_{t-1}) - f^\star + (\tfrac{5 \eta}{2\alpha} - \eta)\E \norm{\bar\ee_{t-1}}^2 + \frac{\eta}{4}\E\norm{\nabla f(\xx_{t-2})}^2}_{=: \xi_{t-1}} \\&\hspace{5mm} - \frac{\eta}{4}\E \norm{\nabla f(\xx_{t-1})}^2 + \frac{5\eta\alpha}{n}\sigma^2 + \frac{5\eta}{4}\E\norm{\mm_{t-1} - \bar\mm_{t-1}}^2\\
                                                                                                                                                         & \leq \xi_{t-1} - \frac{\eta}{4}\E \norm{\nabla f(\xx_{t-1})}^2                                                                                                        \\&\hspace{5mm} +  \frac{5\eta\alpha \sigma^2}{2}\rbr*{\frac{2}{n} + \delta(c + \tfrac{c}{\alpha}(1-\alpha)^{t-2})}\,.
\end{align*}
Further, specializing the descent bound Lemma~\ref{lem:sgdm-byz-descent} and error bound Lemma~\ref{lem:sgdm-byz-error} for $t=1$ we have
\begin{align*}
  \xi_1 & \leq  \E~f(\xx_{1}) - f^\star + \frac{3 \eta}{2}\E \norm{\bar\ee_{1}}^2 + \frac{\eta}{4}\E\norm{\nabla f(\xx_{0})}^2                                \\
        & \leq f(\xx_{0}) - f^\star + \frac{5\eta}{2}\E \norm{\bar\ee_{1}}^2 - \frac{\eta}{4}\E\norm{\nabla f(\xx_{0})}^2 + \eta \E\norm{\mm_1 - \bar\mm_1}^2 \\
        & \leq f(\xx_{0}) - f^\star - \frac{\eta}{4}\E\norm{\nabla f(\xx_{0})}^2 + \frac{5\eta\sigma^2}{n} + 2c\eta\delta \sigma^2\,.
\end{align*}
Summing over $t$ and again rearranging our recursion for $\xi_t$ gives
\begin{align*}
  \frac{1}{T}\sum_{t=1}^T\E \norm{\nabla f(\xx_{t-1})}^2 & \leq \frac{4(f(\xx_{0}) - f^\star)}{\eta T} + \frac{20\sigma^2}{n T} + \frac{8c\delta \sigma^2}{T}                     \\&\hspace{5mm} +  \frac{10\alpha \sigma^2}{T}\sum_{t=1}^T\rbr*{\frac{2}{n} + \delta(c + \tfrac{c}{\alpha}(1-\alpha)^{t-2})}\\
                                                         & \leq \frac{4(f(\xx_{0}) - f^\star)}{\eta T} + \frac{20\sigma^2}{n T} + \frac{8c\delta \sigma^2}{T}                     \\&\hspace{5mm} + \frac{20\alpha \sigma^2}{n} + 10c \delta \alpha \sigma^2 + \frac{10\delta c \alpha\sigma^2}{\alpha^2 T}\\
                                                         & = \frac{4(f(\xx_{0}) - f^\star)}{\eta T} + \frac{20\sigma^2}{n T} + \frac{8c\delta \sigma^2}{T}                        \\&\hspace{5mm} + \frac{160L \eta \sigma^2}{n} + 80 L \delta c \eta \sigma^2 + \frac{5c \delta \sigma^2}{4L\eta T}\\
                                                         & \leq 16 \sqrt{\frac{5 \sigma^2 \rbr*{2 + c\delta n}}{nT} \rbr*{L(f(\xx_0) - f^\star) + \tfrac{5c\delta}{16}\sigma^2} } \\&\hspace{5mm} +
  \frac{32 L (f(\xx_{0}) - f^\star)}{ T} +  \frac{10 c \delta \sigma^2}{T} + \frac{20\sigma^2}{n T} + \frac{8c\delta \sigma^2}{T} \,.
\end{align*}
Substituting the appropriate step-size $\eta = \min\rbr*{\sqrt{\frac{f(\xx_0) - f^\star + \tfrac{5c\delta}{16 L}\sigma^2}{20 L T\sigma^2 \rbr*{\tfrac{2}{n} + c\delta}}}, \frac{1}{8L}}$ finishes the proof of the theorem.
\qed


\section{Proof of Theorem~\ref{thm:mvr-convergence-main} (Momentum based variance reduction)}
We now describe how to modify the momentum method with a small correction term to improve its convergence rate \citep{cutkosky2019momentum}. Starting from a given $\xx_0$ and with $\dd_0 = 0$, $\alpha_1 = 1$, we run the following updates with a sequence of momentum parameters $\alpha_t \in [0,1]$ and step-sizes $\eta_t \geq 0$ for $t\geq 2$
\begin{equation}\label{eqn:mvr-worker}\tag{\sc MVR-Worker}
  \begin{split}
    \dd_{t, i} = \alpha_t \gg_{t, i}(\xx_{t-1}) + (1-\alpha_t)\dd_{t-1, i} + (1-\alpha_t)(\gg_{t, i}(\xx_{t-1}) - \gg_{t, i}(\xx_{t-2}))
  \end{split}
\end{equation}
Note that both $\gg_{t, i}(\xx_{t-1})$ and $\gg_{t, i}(\xx_{t-2})$ here are computed using the same stochastic function (same batch) as indicated by the subscript. The good workers communicate $\dd_{t, i}$ whereas the bad ones send arbitrary vectors. Then, the server performs
\begin{equation}\label{eqn:mvr-server}\tag{\sc MVR-Server}
  \begin{split}
    \dd_{t} &= \cA\rbr*{\dd_{t, 1}, \dots, \dd_{t,n}}\\
    \xx_{t} &= \xx_{t-1} - \eta_t \dd_t\,.
  \end{split}
\end{equation}
Define $\bar\dd_t := \frac{1}{\abs{\cG}}\sum_{j\in\cG}\dd_{t,j}$ and $\bar\ee_t := \bar\dd_t - \nabla f(\xx_{t-1})$. Note that since $\alpha_1 = 1$, the first step can be simplified as $\dd_{1, i} = \gg_{1,i}(\xx_0)$. Here we assume that the stochastic gradient conditioned on all past history is unbiased $\E_t[\gg_{t, i}(\xx_{t-1})] = \nabla f(\xx_{t-1})$ and has bounded variance $\sigma^2$. Further, we assume that the stochastic gradients satisfy $\E\norm{\gg_{t, i}(\xx_{t-1}) - \gg_{t, i}(\xx_{t-2})}^2 \leq L^2 \norm{\xx_{t-1} - \xx_{t-2}}^2$. This is stronger than assuming only that the full gradient $\nabla f$ is Lipschitz.

\begin{lemma}\label{lem:mvr-descent}
  For $\alpha_1 =1$ and any $\alpha_t \in [0,1]$ for $t\geq2$, $\eta_t \leq \frac{1}{L}$, and an $L$-smooth function $f$ we have that $E_1[f(\xx_1)] \leq f(\xx_0) - \frac{\eta_t}{2}\norm{\nabla f(\xx_0)}^2 + \frac{\eta_t^2 L}{2}\sigma^2$ and for $t\geq 2$ with $\bar\ee_t := \bar\dd_t - \nabla f(\xx_{t-1})$ we have
  \[
    \E_t[f(\xx_t)] \leq f(\xx_{t-1}) - \frac{\eta_t}{2}\norm{\nabla f(\xx_{t-1})}^2 + \eta\rbr*{\E_t\norm{\bar\ee_t}^2 + \E_t\norm{\dd_t - \bar\dd_t}^2}\,.
  \]
\end{lemma}
The proof is identical to that of Lemma~\ref{lem:sgdm-descent}.

\begin{lemma}\label{lem:mvr-error}
  Using any momentum and step-sizes such that $1 \geq \alpha \geq 16 L^2 \eta^2$ for $t \geq 2$, we have i) $\E[\ee_t] =0$, and ii) for an $L$-smooth function $f$ that $\E \norm{\bar\ee_{1}}^2 \leq 2\sigma^2/n$ and for $t \geq 2$
  \[
    \E \norm{\bar\ee_{t}}^2 \leq  (1- \tfrac{\alpha}{2})\E \norm{\bar\ee_{t-1}}^2 + 8L^2\eta^2\norm{\nabla f(\xx_{t-2})}^2 + 2\alpha^2 \sigma^2/n + 4L^2\eta^2\E \norm{\dd_{t-1} - \bar\dd_{t-1}}^2\,.
  \]
\end{lemma}
\begin{proof}
  Starting from the definition of $\bar\ee_{t+1}$ and $\bar\dd_{t}$,
  \begin{align*}
    \bar\ee_{t} & = \bar\dd_t - \nabla f(\xx_{t-1})                                                                                                       \\
                & = \alpha \bar\gg_t(\xx_{t-1}) + (1-\alpha)\bar\dd_{t-1} + (1-\alpha)(\bar\gg_t(\xx_{t-1}) - \bar\gg_t(\xx_{t-2})) - \nabla f(\xx_{t-1}) \\
                & =\underbrace{(1-\alpha)(\bar\dd_{t-1} - \nabla f(\xx_{t-2}))}_{\cT_1} +                                                                 \\&\hspace{2cm} \underbrace{\alpha (\bar\gg_t(\xx_{t-1}) - \nabla f(\xx_{t-1}))}_{\cT_2} + \underbrace{(1-\alpha)(\bar\gg_t(\xx_{t-1}) - \bar\gg_t(\xx_{t-2}) - \nabla f(\xx_{t-1}) + \nabla f(\xx_{t-2}))}_{\cT_3}\,.
  \end{align*}
  Note that $\cT_1 = (1- \alpha)\bar\ee_{t-1}$ and that $\E[\cT_2] =0, \E[\cT_3]=0$. This proves that $\E[\bar\ee_{t}] = 0$. Further, conditioned on all history $\cF_t$ (i.e. everything before step $t$), we have $\E_t[\cT_2] = 0$ and $\E_t[\cT_3] =0$ and $\cT_1$ is deterministic. Hence, we can take squared norms on both sides as expand as
  \begin{align*}
    \E \norm{\bar\ee_{t}}^2 & = (1-\alpha)^2\E\norm{\bar\ee_{t-1}}^2 +                                                                                                                                                                    \\&\hspace{1.5cm} \E\norm{\alpha (\bar\gg_t(\xx_{t-1}) - \nabla f(\xx_{t-1})) + (1-\alpha)(\bar\gg_t(\xx_{t-1}) - \bar\gg_t(\xx_{t-2}) - \nabla f(\xx_{t-1}) + \nabla f(\xx_{t-2}))}^2\\
                            & \leq (1-\alpha)^2\E\norm{\bar\ee_{t-1}}^2 +                                                                                                                                                                 \\&\hspace{1.5cm} 2\E\norm{\alpha (\bar\gg_t(\xx_{t-1}) - \nabla f(\xx_{t-1}))}^2 + 2\norm{(1-\alpha)(\bar\gg_t(\xx_{t-1}) - \bar\gg_t(\xx_{t-2}) - \nabla f(\xx_{t-1}) + \nabla f(\xx_{t-2}))}^2\\
                            & \leq (1-\alpha)\E\norm{\bar\ee_{t-1}}^2 + 2\alpha^2 \sigma^2/n + 2(1 - \alpha)^2\E \norm{\bar\gg_t(\xx_{t-1}) - \bar\gg_t(\xx_{t-2})}^2                                                                     \\
                            & \leq (1-\alpha)\E\norm{\bar\ee_{t-1}}^2 + 2\alpha^2 \sigma^2/n + 2(1 - \alpha)^2 L^2 \E \norm{\xx_{t-1} - \xx_{t-2}}^2                                                                                      \\
                            & = (1-\alpha)\E\norm{\bar\ee_{t-1}}^2 + 2\alpha^2 \sigma^2/n + 4(1 - \alpha)^2 L^2 \eta^2 \E \norm{\bar\dd_{t-1}}^2 + 4(1 - \alpha)^2 L^2 \eta^2 \E \norm{\dd_{t-1} - \bar\dd_{t-1}}^2                       \\
                            & \leq (1-\alpha)\E\norm{\bar\ee_{t-1}}^2 + 2\alpha^2 \sigma^2/n + 8L^2 \eta^2 \E \norm{\bar\ee_{t-1}}^2 + 8L^2 \eta^2 \E \norm{\nabla f(\xx_{t-2})}^2 + 4L^2 \eta^2 \E \norm{\dd_{t-1} - \bar\dd_{t-1}}^2\,.
  \end{align*}

  Here the the third inequality used the expected squared Lipschitzness of $g_t(\,\cdot\,)$, whereas the rest relied on Young's inequality and that $\alpha \in [0,1]$. Now the condition on the momentum implies that $8L^2 \eta^2 \leq \frac{\alpha}{2}$, yielding the second statement of the lemma for $t\geq 2$. The statement for $\ee_1$ follows since $\bar\dd_0 = 0$.
\end{proof}

\begin{lemma}[Aggregation error]\label{lem:mvr-agg-error}
  Given Definition~\ref{asm:agg-ass} holds and we use a momentum constant parameter $\alpha_1 = 1$ and $\alpha_t = \alpha \geq 192 L^2\eta^2(c \delta + 1)$ for $t \geq 2$, the error between the ideal average momentum $\bar\dd_t$ and the robust aggregate $\dd_t$ for any $t\geq 2$ can be bounded as
  \begin{align*}
    \E \norm{\bar\ee_{t}}^2 + c\delta\E\norm{\dd_{i, t} - \dd_{j, t}}^2 & \leq (1- \tfrac{\alpha}{4})\rbr*{\E \norm{\bar\ee_{t-1}}^2 + c\delta\E\norm{\dd_{i, t-1} - \dd_{j, t-1}}^2}
    \\&\hspace*{2cm}+ \tfrac{\alpha}{16}\norm{\nabla f(\xx_{t-2})}^2 + (c\delta + 1/n)4\alpha^2 \sigma^2
  \end{align*}
  For $t=1$, we can simplify the bound to $\E \norm{\bar\ee_{1}}^2 + c\delta\E\norm{\dd_{i, 2} - \dd_{j, 2}}^2 \leq 2\sigma^2(c\delta  + 1/n)$.
\end{lemma}
\begin{proof}
  Expanding the definition of the worker momentum for any two good workers $i, j \in \cG$ for $t\geq 2$,
  \begin{align*}
    \E\norm{\dd_{i, t} - \dd_{j, t}}^2 & =
    \E\| \alpha (\gg_i(\xx_{t-1}) - \gg_j(\xx_{t-1})) +
    \\ &\hspace*{2cm} (1 - \alpha)(\dd_{i, t-1} - \dd_{j, t-1}) +
    \\&\hspace*{2cm} (1-\alpha)(\gg_{t, i}(\xx_{t-1}) - \gg_{t, j}(\xx_{t-1}) - \gg_{t, i}(\xx_{t-2}) + \gg_{t, j}(\xx_{t-2})) \|^2 \\
                                       & \leq \E\norm{(1 - \alpha)(\dd_{i, t-1} - \dd_{j, t-1})}^2 + 4\alpha^2 \sigma^2 + 4L^2(1-\alpha)^2 \E\norm{\xx_{t-1} - \xx_{t-2}}^2 \\
                                       & \leq (1 - \alpha)\E\norm{\dd_{i, t-1} - \dd_{j, t-1}}^2 + 4\alpha^2 \sigma^2 + 4L^2\eta^2\E\norm{\dd_{t-1}}^2                      \\
                                       & \leq (1 - \alpha)\E\norm{\dd_{i, t-1} - \dd_{j, t-1}}^2 + 4\alpha^2 \sigma^2 + 12L^2\eta^2\E\norm{\dd_{t-1} - \bar\dd_{t-1}}^2
    \\ &\hspace*{2cm}+ 12L^2\eta^2\E\norm{\bar\ee_{t-1}}^2 + 12L^2\eta^2\E\norm{\nabla f(\xx_{t-2})}^2\\
                                       & \leq (1 - \alpha + 12c \delta L^2\eta^2)\E\norm{\dd_{i, t-1} - \dd_{j, t-1}}^2 + 4\alpha^2 \sigma^2
    \\ &\hspace*{2cm}+ 12L^2\eta^2\E\norm{\bar\ee_{t-1}}^2 + 12L^2\eta^2\E\norm{\nabla f(\xx_{t-2})}^2
    \,.
  \end{align*}

  Scale this by $c\delta$ and then add the inequality from Lemma~\ref{lem:mvr-error} to get
  \begin{align*}
    \E \norm{\bar\ee_{t}}^2 + c\delta\E\norm{\dd_{i, t} - \dd_{j, t}}^2 & \leq
    (1- \tfrac{\alpha}{2})\E \norm{\bar\ee_{t-1}}^2 + 8L^2\eta^2\norm{\nabla f(\xx_{t-2})}^2 + 2\alpha^2 \sigma^2/n + 4L^2\eta^2\E \norm{\dd_{t-1} - \bar\dd_{t-1}}^2                                                                                                         \\
                                                                        & \hspace*{1cm}+ (1 - \alpha + 12c \delta L^2\eta^2)c\delta\E\norm{\dd_{i, t-1} - \dd_{j, t-1}}^2 + 4c\delta\alpha^2 \sigma^2
    \\ &\hspace*{2cm}+ 12c\delta L^2\eta^2\E\norm{\bar\ee_{t-1}}^2 + 12c\delta L^2\eta^2\E\norm{\nabla f(\xx_{t-2})}^2\\
                                                                        & \leq
    (1- \tfrac{\alpha}{2})\E \norm{\bar\ee_{t-1}}^2 + 8L^2\eta^2\norm{\nabla f(\xx_{t-2})}^2 + 2\alpha^2 \sigma^2/n                                                                                                                                                           \\
                                                                        & \hspace*{1cm}+ (1 - \alpha + 12c \delta L^2\eta^2 + 4L^2\eta^2)c\delta\E\norm{\dd_{i, t-1} - \dd_{j, t-1}}^2 + 4c\delta\alpha^2 \sigma^2
    \\ &\hspace*{2cm}+ 12c\delta L^2\eta^2\E\norm{\bar\ee_{t-1}}^2 + 12c\delta L^2\eta^2\E\norm{\nabla f(\xx_{t-2})}^2\\
                                                                        & = (1- \tfrac{\alpha}{2} + 12c \delta L^2\eta^2)\E \norm{\bar\ee_{t-1}}^2 + (8L^2\eta^2 + 12c \delta L^2\eta^2)\norm{\nabla f(\xx_{t-2})}^2 + (4c\delta + 2/n)\alpha^2 \sigma^2                      \\
                                                                        & \hspace*{1cm}+ (1 - \alpha + 12c \delta L^2\eta^2 + 4L^2\eta^2)c\delta\E\norm{\dd_{i, t-1} - \dd_{j, t-1}}^2                                                                                        \\
                                                                        & \leq (1- \tfrac{\alpha}{4})\rbr*{\E \norm{\bar\ee_{t-1}}^2 + c\delta\E\norm{\dd_{i, t-1} - \dd_{j, t-1}}^2} + \tfrac{\alpha}{16}\norm{\nabla f(\xx_{t-2})}^2 + (4c\delta + 2/n)\alpha^2 \sigma^2\,.
  \end{align*}
  Here we used $\alpha \geq 192 L^2\eta^2(c \delta + 1)$, and
  Definition~\ref{asm:agg-ass} that $\E\norm{\dd_{t-1} - \bar\dd_{t-1}}^2 \leq c\delta \E\norm{\dd_{i, t-1} - \dd_{j, t-1}}^2$.

\end{proof}

We are now ready to prove the convergence theorem.
\begin{theorem}[Byzantine robust MVR]\label{thm:mvr-convergence}
  Let us run the MVR algorithm combined with a robust aggregation rule \cA with step-size $\eta = \min\rbr*{\sqrt[3]{\frac{f(\xx_0) - f^\star}{T(1536L^2 \sigma^2(c\delta +1)(c\delta + 1/n))}}, \frac{1}{4L}}$ and momentum parameter $\alpha = 192 L^2 \eta^2(1+c\delta)$. Then,
  \[
    \frac{1}{T}\sum_{t=1}^T \E \norm{\nabla f(\xx_{t-1})}^2 \leq  120 \rbr*{\frac{L \sigma\sqrt{(c\delta + 1/n)(c\delta+1)}(f(\xx_0) - f^\star)}{T}}^{\frac{2}{3}} + \frac{16L(f(\xx_0) - f^\star) + 32\sigma^2(c\delta + 1/n)}{T}\,.
  \]
\end{theorem}
\begin{proof}
  Scaling Lemma~\ref{lem:mvr-agg-error} by $\frac{4\eta}{\alpha}$ and adding it to Lemma~\ref{lem:mvr-descent} we have for any $t \geq 2$
  \begin{align*}
    (\E~f(\xx_t) - f^\star) + \frac{4\eta}{\alpha} \rbr*{\E \norm{\bar\ee_{t}}^2 + c\delta\E\norm{\dd_{i, t} - \dd_{j, t}}^2} & \leq (\E~f(\xx_{t-1}) - f^\star) + \frac{4\eta}{\alpha} \rbr*{\E \norm{\bar\ee_{t-1}}^2 + c\delta\E\norm{\dd_{i, t-1} - \dd_{j, t-1}}^2}     \\
                                                                                                                              & \hspace{1cm} -\eta\rbr*{\E \norm{\bar\ee_{t-1}}^2 + c\delta\E\norm{\dd_{i, t-1} - \dd_{j, t-1}}^2}                                           \\
                                                                                                                              & \hspace*{1cm} -\frac{\eta}{2}\E\norm{\nabla f(\xx_{t-1})}^2 + \eta\rbr*{\E \norm{\bar\ee_{t}}^2 + c\delta\E\norm{\dd_{t, i} - \dd_{t, j}}^2} \\
                                                                                                                              & \hspace*{1cm} + \tfrac{\eta}{4}\E\norm{\nabla f(\xx_{t-2})}^2 + (c\delta + 1/n)16\eta\alpha \sigma^2\,.
  \end{align*}
  Define the constant
  \[
    \xi_t := (\E~f(\xx_t) - f^\star) + \rbr*{\frac{4\eta}{\alpha} -\eta} \rbr*{\E \norm{\bar\ee_{t}}^2 + c\delta\E\norm{\dd_{i, t} - \dd_{j, t}}^2} + \frac{\eta}{4}\E\norm{\nabla f(\xx_{t-1})}^2\,.
  \]
  Then the previously stated inequality can be rearranged as
  \[
    \frac{\eta}{4}\E\norm{\nabla f(\xx_{t-1})}^2  \leq  \xi_{t-1} - \xi_t + (c\delta + 1/n)16\eta\alpha \sigma^2\,.
  \]
  Also note that $\xi_t \geq 0$ for any $t$ and also for $t=1$,
  \begin{align*}
    \xi_1 & =  \E~f(\xx_{1}) - f^\star + \rbr*{\frac{4\eta}{\alpha} -\eta} \rbr*{\E \norm{\bar\ee_{t}}^2 + c\delta\E\norm{\dd_{i, t} - \dd_{j, t}}^2} + \frac{\eta}{4}\E\norm{\nabla f(\xx_{0})}^2 \\
          & \leq f(\xx_0) - f^\star - \frac{\eta}{4} \E \norm{\nabla f(\xx_{0})}^2 + 8\eta\sigma^2(c\delta  + 1/n)\,.
  \end{align*}
  Note that here we assumed a batch size of $T$ in the first step to simplify computations. This does not change the asymptotic rate (multiplies it by 2), similar to \citep{tran2020hybrid}. This is easy to work around by using changing step-sizes/momentum values as shown by \citep{cutkosky2019momentum}. Now summing over $t$ and again rearranging gives
  \begin{align*}
    \frac{1}{\sum_{t=1}^\ell \eta}\sum_{t=1}^\ell \eta \E \norm{\nabla f(\xx_{t-1})}^2 \leq \frac{4(f(\xx_0) - f^\star)}{\sum_{t=1}^\ell \eta} + \frac{1}{\sum_{t=1}^\ell \eta} \sum_{t=1}^\ell 32(c\delta  + 1/n) \eta \alpha \sigma^2\,.
  \end{align*}
  For simplicity, let us use a constant $\eta = \min\rbr*{\sqrt[3]{\frac{f(\xx_0) - f^\star}{T(1536L^2 \sigma^2(c\delta +1)^2)}}, \frac{1}{4L}}$ for $t \geq 1$ and momentum parameter $\alpha_1 = 1$ and $\alpha = 192L^2 \eta^2(c\delta +1)$ for $t\geq 2$. This simplifies the above inequality to
  \[
    \frac{1}{T}\sum_{t=1}^T \E \norm{\nabla f(\xx_{t-1})}^2 \leq \frac{4(f(\xx_0) - f^\star)}{\eta T} + 6144 L^2 \eta^2(c\delta +1)(c\delta +1/n) \sigma^2 + \frac{32\sigma^2(c\delta + 1/n)}{T}\,.
  \]
  Substituting the appropriate $\eta$ yields the desired rate.
\end{proof}

\section{Additional Experiments} \label{sec:additional_exps}
\subsection{Experiment setups} \label{ssec:setups}
\subsubsection{General setup}
The default experiment setup is listed in \Cref{tab:setup:default}. The default hyperparameters of the aggregators are summarized as follows

  {
    \hfill
    \begin{tabular}{ll}
      \toprule
      Aggregators & Hyperparameters \\
      \midrule
      Krum        & N/A             \\
      CM          & N/A             \\
      RFA         & $T=3$           \\
      TM          & $b=\delta$      \\
      CC          & $\tau=100$      \\
      \bottomrule
    \end{tabular}
    \hfill
  }

\paragraph{About \Cref{fig:exp3}.} We have the following setup
\begin{itemize}[nosep]
  \item For all aggregators except mean, there are $n=25$ workers and $n\delta=11$ of them are Byzantine.
  \item For aggregator mean, there are $n=14$ workers and 0 Byzantine workers.
  \item The IPM attack has strength of $\epsilon=0.1$.
  \item The ALIE Attack has a hyperparameter $z$ which is computed according to  \cite{baruch2019little}
        $$z=\max_{z} \left(\phi(z) < \frac{n-n\delta-s}{n-n\delta} \right)$$
        where $s=\lfloor \frac{n}{2}+1\rfloor-n\delta$ and $\phi$ is the cumulative standard normal function. In our setup, the $z\approx 1.06$.

\end{itemize}

\begin{table}[h]
  \caption{Default experimental settings for CIFAR-10 and MNIST.}
  \centering
  \small
  \label{tab:setup:default}%
  \begin{tabular}{lll}
    \toprule
    Dataset                   & CIFAR-10                     & MNIST                           \\
    Architecture              & ResNet-20  \cite{he2016deep} & CONV-CONV-DROPOUT-FC-DROPOUT-FC \\
    Training objective        & Cross entropy loss           & Negative log likelihood loss    \\
    Evaluation objective      & Top-1 accuracy               & Top-1 accuracy                  \\
    \midrule
    Batch size per worker     & $32$                         & 1                               \\
    Momentum       $\beta$    & 0 or 0.9 or 0.99             & 0                               \\
    Learning rate             & 0.1                          & $\frac{0.1}{256}$               \\
    LR decay                  & 0.1 at epoch 75              & No                              \\
    LR warmup                 & No                           & No                              \\
    \# Epochs / \# Iterations & 100 Epochs                   & 800 Iterations                  \\
    Weight decay              & No                           & No                              \\
    \midrule
    Repetitions               & 2, with varying seeds        & 2, with varying seeds           \\
    \bottomrule
  \end{tabular}
\end{table}

\subsubsection{Constructing datasets}

\paragraph{Long-tailness.} The MNIST dataset has 10 classes each with similar amount of samples. The long-tailness  is achieved by sampling class with exponentially decreasing portions $\gamma\in(0, 1]$. That is, for class $i=1,\ldots,10$, we only randomly sample $\gamma^i$ portion of all samples in class $i$. Note that the same procedure has to be applied to the test dataset.

\paragraph{About dataset on Byzantine workers.} The training set is divided by the number of good workers. So the good workers has to full information of training dataset. The Byzantine worker has access to the whole training dataset.

\subsubsection{Running environment}
We summarize the running environment of this paper as in \Cref{tab:runtime}.
\begin{table}[!h]
  \caption{Runtime hardwares and softwares.}
  \centering
  \label{tab:runtime}%
  \begin{tabular}{p{0.22\linewidth} p{0.7\linewidth}}
    \toprule
    CPU                       &                                             \\
    \hspace{1em} Model name   & Intel (R) Xeon (R) Gold 6132 CPU @ 2.60 GHz \\
    \hspace{1em} \# CPU(s)    & 56                                          \\
    \hspace{1em} NUMA node(s) & 2                                           \\ \midrule
    GPU                       &                                             \\
    \hspace{1em} Product Name & Tesla V100-SXM2-32GB                        \\
    \hspace{1em} CUDA Version & 11.0                                        \\ \midrule
    PyTorch                   &                                             \\
    \hspace{1em} Version      & 1.7.1                                       \\
    \bottomrule
  \end{tabular}
\end{table}

\subsection{Exploring local steps between aggregations}
\begin{figure}[H]
  \vspace{-3mm}
  \centering
  \includegraphics[
    width=0.5\linewidth
  ]{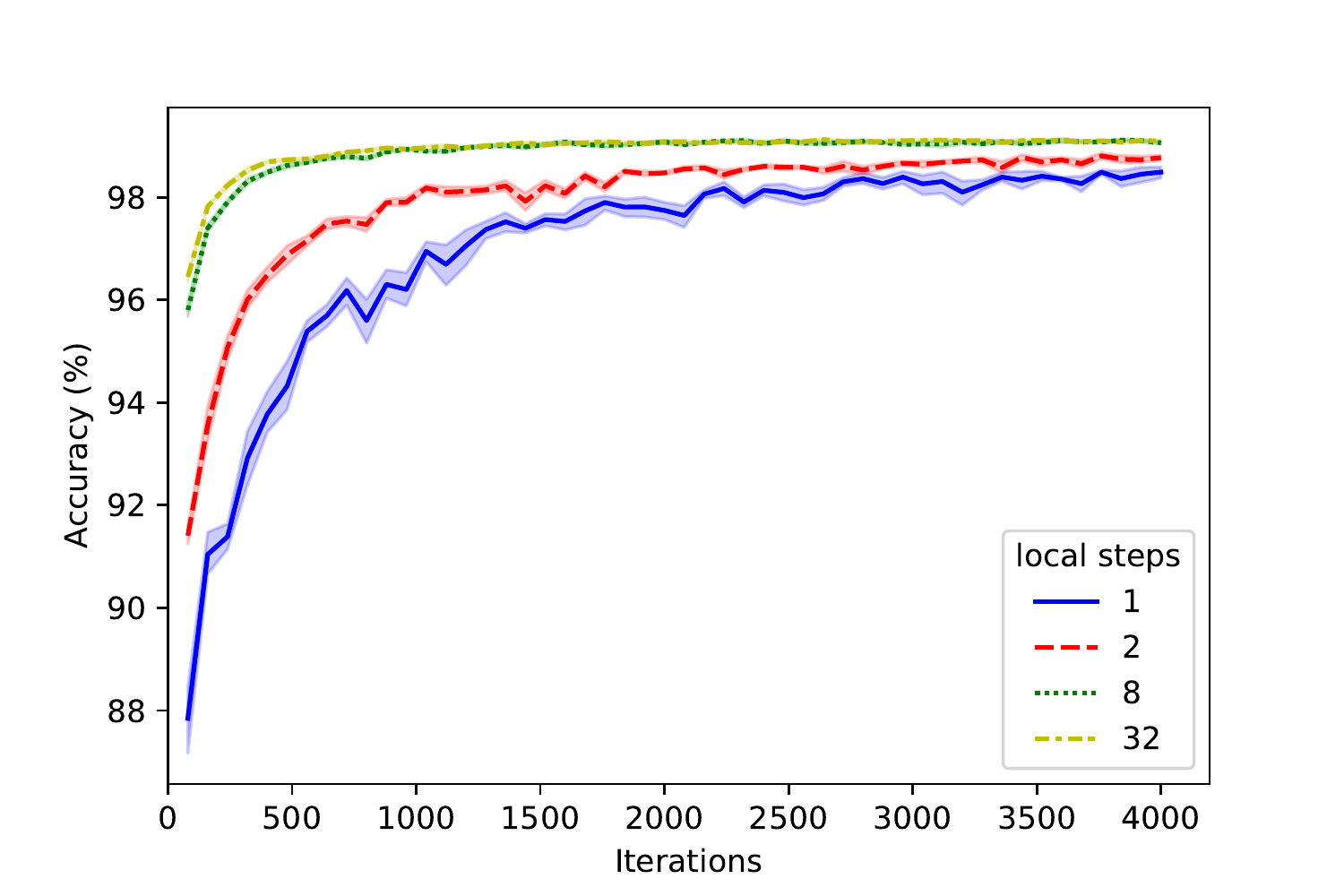}
  \caption{\ref{eqn:iter-clip} with 1, 2, 8, 32, local steps for MNIST dataset.}
  \label{fig:exp4}
\end{figure}

In this experiment, we combine \ref{eqn:iter-clip} with local SGD and bench marked on MNIST without attacker. The results in Fig.~\ref{fig:exp4} shows that using higher local steps improves the accuracy and convergence rate. It supports that \ref{eqn:iter-clip} is compatible with localSGD.

\subsection{Comparison with \citep{allenzhu2021byzantineresilient}
} \label{subsec:app-allenzhu-comp}
\begin{figure}[h]
  \vspace{-3mm}
  \centering
  \includegraphics[
    width=0.5\linewidth
  ]{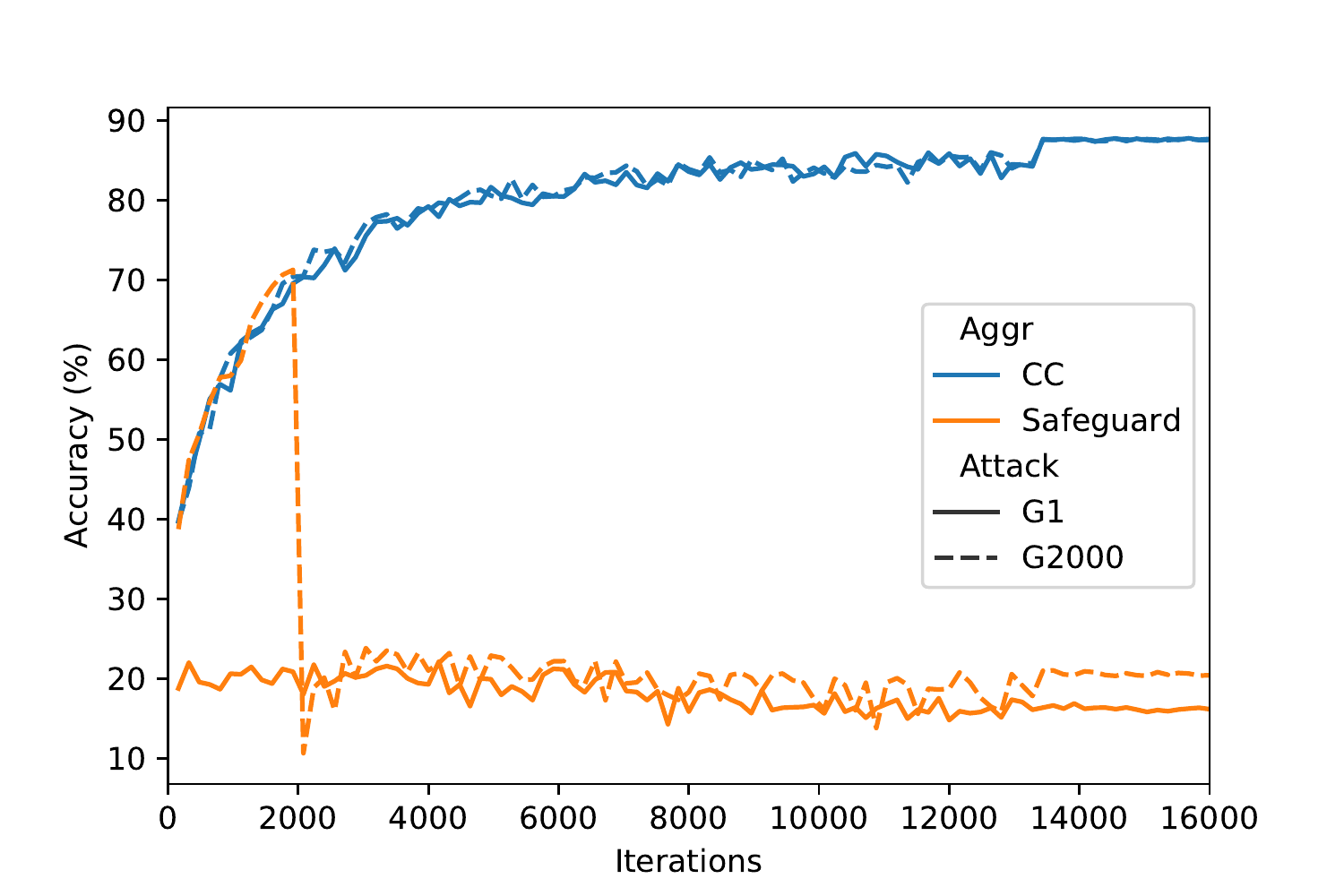}
  \caption{Comparing \ref{eqn:iter-clip} ($\tau=100$) with Safeguard \citep{allenzhu2021byzantineresilient} ($T_0=1$,$T_1=6$, $\mathfrak{T}_0=20$, $\mathfrak{T}_1=50$). The Byzantine workers send to the server vectors from a Gaussian distribution with standard deviation of $10^8$. The ``$G1$'' attack inject attack at the 1st iteration while the ``G2000'' attack inject attack at the 2000th iteration. There are 10 nodes in total and 4 of them are Byzantine. The underlying dataset is Cifar10. We use batch size 32 and learning rate 0.1.}
  \label{fig:gaussian}
\end{figure}

The recent independent work \textsc{Safeguard} \citep{allenzhu2021byzantineresilient} also uses historical information to detect Byzantine workers. However, as we discussed earlier, they assume that the noise in stochastic gradients is bounded almost surely instead of the more standard assumption that only the variance is bounded. Theoretically, such strong assumptions are unlikely to hold~\citep{zhang2019adam} and even Gaussian noise is excluded. Further, the lower-bounds of \citep{arjevani2019lower} no longer apply, and thus their algorithm may be sub-optimal. Practically, their algorithm removes suspected workers either permanently (a decision of high risk), or resets the list of suspects at each window boundary (which is sensitive to the choice of hyperparameters). Having said that, \cite{allenzhu2021byzantineresilient} prove convergence to a local minimum instead of to a saddle point as we do here. In this secion, we conduct further empirical comparison of Safeguard and Centered clip \ref{eqn:iter-clip}.

First, note that Algorithm 1 in \citep{allenzhu2021byzantineresilient} is vulnerable to simple attacks, e.g. sending an arbitrary vector of very large magnitude, while \ref{eqn:iter-clip} is not. This is because Safeguard uses information from the previous step to filter in the current step. This is necessary in order to make the algorithm amenable to analysis. This means that even if a Byzantine worker sends a very large bad update, the algorithm will apply it once and filter out the worker only from the next round onward. Thus, all Byzantine workers can ensure that their update is incorporated at least once. While theoretically this might not be problematic since the influence of a single update is limited, in practice this means that the Byzantine workers can push the training process to encounter \texttt{NaN}s, ensuring no chance of recovery.

To demonstrate the effect, we apply the Gaussian attack to \textsc{Safeguard} and \ref{eqn:iter-clip} at $t=1$ (G1) and $t=2000$ (G2000). The Gaussian attacker sends to the server vectors of Gaussian distribution of standard deviation $10^8$. Since the workers behave correctly until $t-1$, they all belong to $\textbf{good}_t$ and their updates are incorporated. While the Byzantine worker is removed from $\textbf{good}_{t+1}$, the attack already succeeded and there is no chance of recovery.
We show the experimental results in Figure~\ref{fig:gaussian}. In contrast, \ref{eqn:iter-clip} (even without momentum) easily defends against such attacks.


Secondly, \textsc{Safeguard} requires tuning additional parameters (e.g. $\mathfrak{T}_0$, $\mathfrak{T}_1$) for each kind of attack while \ref{eqn:iter-clip} does not. For example, \textsc{Safeguard} uses $\mathfrak{T}_0=1$, $\mathfrak{T}_1=2$ for Bit-Flipping attack \citep[Appendix C.2.1]{allenzhu2021byzantineresilient} and $\mathfrak{T}_0=2$, $\mathfrak{T}_1=7$ for Label-Flipping attack \citep[Appendix C.2.3]{allenzhu2021byzantineresilient}.
However, by the definition of Byzantine attack, the attacker is allowed to adaptively change attacks \emph{after} tuning. This makes it crucial to ensure that any Byzantine robust algorithm works without additional tuning. In contrast, \ref{eqn:iter-clip} uses $\tau=100$ and $l=1$ for all experiments in the paper unless otherwise clarified.

\clearpage
\end{document}